\documentclass[twoside]{article}

\usepackage[accepted]{aistats2026}
\usepackage[round]{natbib}

\usepackage{hyperref}
\usepackage{xcolor}
\usepackage{algorithm}
\usepackage{algpseudocode}
\usepackage{booktabs,subcaption,threeparttable}
\usepackage{graphicx}
\usepackage{caption}
\usepackage{subcaption}
\usepackage{tikz}
\usetikzlibrary{automata}
\usetikzlibrary{positioning}  
\usetikzlibrary{arrows}       
\tikzset{
  node distance=1.5cm,
  every state/.style={
    semithick,
    fill=gray!5,
    minimum size=10mm, % <-- same diameter for all states
    inner sep=0pt,     % <-- no extra padding
    align=center
  },
  initial text={},
  double distance=3pt,
  every edge/.style={draw,->,>=stealth',auto,semithick}
}

\usepackage{amsmath}
\usepackage{amssymb}
\usepackage{amsthm}
\usepackage{mathtools}

\theoremstyle{definition}
\newtheorem{definition}{Definition}[section]
\newtheorem{assumption}{Assumption}[section]
\theoremstyle{plain}
\newtheorem{theorem}{Theorem}[section]

\newtheorem{lemma}{Lemma}[section]
\newtheorem{corollary}{Corollary}[section]

\newcommand{\bftab}{\fontseries{b}\selectfont}
\newcommand\dunderline[2][.4pt]{%
  \raisebox{-#1}{\underline{\raisebox{#1}{\smash{\underline{#2}}}}}}

\begin{document}



\twocolumn[

\aistatstitle{Counterfactually Fair Conformal Prediction}

\aistatsauthor{Ozgur Guldogan \And Neeraj Sarna \And  Yuanyuan Li \And Michael Berger}

\aistatsaddress{UC Santa Barbara \And  Munich RE \And Munich RE  \And Munich RE} ]

\begin{abstract}
While counterfactual fairness of point predictors is well studied, its extension to prediction \emph{sets}---central to fair decision-making under uncertainty---remains underexplored. 
On the other hand, conformal prediction (CP) provides efficient, distribution-free, finite-sample valid prediction sets, yet does not ensure counterfactual fairness.
We close this gap by developing \textit{Counterfactually Fair Conformal Prediction} (CF-CP) that produces counterfactually fair prediction sets. 
Through symmetrization of conformity scores across protected-attribute interventions, we prove that CF-CP results in counterfactually fair prediction sets while maintaining the marginal coverage property. 
Furthermore, we empirically demonstrate that on both synthetic and real datasets, across regression and classification tasks, CF-CP achieves the desired counterfactual fairness and meets the target coverage rate with minimal increase in prediction set size.
CF-CP offers a simple, training-free route to counterfactually fair uncertainty quantification.
\end{abstract}

\section{INTRODUCTION}
\label{sec:intro}
As machine learning models are being increasingly used in high-stakes applications like healthcare, criminal justice, test score estimation, etc.; it is crucial to quantify and mitigate the associated bias \citep{barocas2016big,podesta2014big,corbett2023measure}. 
This problem has been extensively explored in the past for both regression and classification problems \citep{berk2018fairness,chouldechova2017fair,agarwal2019fair}. 
The general idea is to ensure that subpopulations are treated equally (group fairness) or focus on individual fairness such as counterfactual fairness~\citep{kusner2017counterfactual}. 
Although these fairness mitigation techniques are effective, most of them focus on point predictions and fail to capture the inherent uncertainty in ML models.
Moreover, recent work~\citep{cresswell2024conformal} has shown that conformal prediction sets can improve decision quality in high-stakes tasks, underscoring the practical importance of uncertainty-aware prediction sets beyond point predictors.

An appealing methodology that promotes model fairness under uncertainty is the idea of \textit{fair prediction sets} \citep{romano2020malice,liu2022conformalized}. 
These methods provide marginal coverage guarantees while enforcing a group-level fairness criterion. 
Typically, they adapt classical group-fairness metrics for point predictors (e.g., demographic parity) to the set-valued setting via inclusion-based definitions \citep{vadlamani2025a}.
To the best of our knowledge, all these past works on fair prediction sets have focused on group-level fairness.
Furthermore, recent human-subject studies~\citep{cresswell2025conformal} suggest that enforcing certain group-fairness criteria for prediction sets, such as equalized coverage, may inadvertently increase disparate impact, highlighting the need for alternative fairness notions for conformal prediction beyond group-level criteria.
Thus, although group-level fairness is a powerful tool, certain applications demand individual-level guarantees.
Consider healthcare, for instance, where the intention could be to ensure a patient’s recommended treatment is not systematically influenced by protected attributes; such concerns have been widely discussed in the context of fair machine learning for laboratory medicine~\citep{azimi2023optimizing}.
A method can satisfy group-level parity yet still produce biased recommendations for particular individuals, which could be undesirable for such an application.
Furthermore, as noted by~\citet{kusner2017counterfactual}, the notion of counterfactual fairness better captures the causality of bias in ML algorithms.

The above motivates prediction sets with individual-level fairness. 
We extend counterfactual fairness notion to set-valued predictors: for a given individual, the prediction set should remain unchanged under counterfactual interventions on the protected attribute. 
For example, consider constructing a prediction set for a particular student’s SAT score; a counterfactually fair prediction set would be identical for that student whether we intervene to change their protected attribute (e.g., gender) or not.

We focus on conformal prediction (CP) and develop a framework called \emph{Counterfactually Fair Conformal Prediction} (CF-CP). 
Concretely, we compute conformity scores under each intervention on the protected attribute and then aggregate these per-intervention scores using a symmetric function to obtain a counterfactually fair score. 
We then use these scores with the standard split CP procedure.
We show that under an invertible structural causal model (SCM) and the exchangeability assumptions, CF-CP guarantees both marginal coverage and set-level counterfactual fairness. 
We emphasize that the latter holds even when the underlying point predictor is not counterfactually fair.
Empirically, we evaluate CF-CP on synthetic and real-world datasets, showing that it significantly reduces unfairness in the prediction sets compared to standard split CP, with only a modest increase in average set size.

We note that prediction sets generated by CP are not inherently fair. They can systematically differ between protected groups or change when we hypothetically alter an individual's sensitive attributes \citep{romano2020malice,vadlamani2025a}. 
CF-CP alters the scoring function of standard CP such that counterfactual fairness and coverage is ensured.

Our key contributions are summarized as follows:
\begin{itemize}
    \item We extend the notion of counterfactual fairness to set-valued predictors and propose a post-training conformal prediction procedure (CF-CP) that symmetrizes conformity scores across protected attribute interventions.
    \item We prove that under an invertible SCM and exchangeability assumptions, CF-CP guarantees both marginal coverage and set-level counterfactual fairness.
    \item We empirically demonstrate that CF-CP effectively reduces counterfactual unfairness in prediction sets on synthetic and real-world datasets while maintaining desired coverage.
\end{itemize}

\subsection{Related Work}
\label{sec:related-work}

Fairness in conformal prediction has primarily focused on group fairness notions~\citep{romano2020malice}.
\citet{vadlamani2025a} propose a method that adapts common group fairness notions (e.g., demographic parity, equalized odds) to conformal prediction without requiring group membership during test time.
One line of work tailors CP-based regression to group-fairness goals, specifically Demographic Parity and Equal Opportunity \citep{liu2022conformalized,wang2023equal}. 
A separate line advances CP methods that seek equalized coverage across groups \citep{lu2022fair}.
Our work differs in that we focus on individual-level counterfactual fairness for set-valued predictors.
Moreover, recent human-subject studies by~\citet{cresswell2025conformal} show that enforcing equalized coverage across groups can inadvertently increase disparate impact in decision-making, further underscoring the need for alternative fairness notions for prediction sets.
There is also a growing body of work that combines causality with uncertainty quantification, though not in the context of fairness.
\citet{lei2021conformal, schroder2024conformal, chen2024conformal} study conformal prediction for potential outcomes in causal inference.
However, these methods do not consider fairness.
We next summarize counterfactual fairness methods for point predictors.

Counterfactual fairness (CF), introduced by \citet{kusner2017counterfactual}, is an individual-level fairness notion that requires a predictor's output to remain invariant under counterfactual changes to protected attributes.
There are several strategies to design CF predictors.
\citet{kusner2017counterfactual} propose to learn a predictor only on the features that are not descendants of the protected attribute $A$ in the causal graph.
\citet{zuo2023counterfactually} propose a method that achieves CF by mixing representations of factual and counterfactual instances before feeding to a predictor.
\citet{zhou2024counterfactual} propose a method that combines both factual and counterfactual predictions weighted by the probability of the corresponding protected attribute value.
\citet{di2020counterfactual,garg2019counterfactual} study regularizers that penalize differences in predictions between factual and counterfactual instances.
These methods focus on point predictors, while we extend CF to set-valued predictors and design a conformal prediction procedure that enforces this property.

\paragraph{Paper organization.}
The rest of the paper is organized as follows. 
In Section~\ref{sec:background}, we review background on conformal prediction. 
Section~\ref{sec:problem} introduces counterfactual fairness, first for point predictors and then extended to prediction sets. 
In Section~\ref{sec:method}, we present our CF-CP method, discuss its computational properties, and establish theoretical guarantees. 
Section~\ref{sec:experiments} reports empirical results on synthetic and real datasets. 
We conclude in Section~\ref{sec:conclusion}.
\section{BACKGROUND: CONFORMAL PREDICTION}
\label{sec:background}

\paragraph{Notation and Setup.}
We consider observed features $X \in \mathcal{X}$, a protected attribute $A \in \mathcal{A}$, and a label $Y \in \mathcal{Y}$. 
Throughout the paper, we assume $\mathcal{A}$ is finite (e.g., binary or categorical).
We write a (possibly randomized) predictor as $\hat{f}:\mathcal{X}\times\mathcal{A}\to\mathcal{S}$, where $\mathcal{S}=\mathbb{R}$ for regression or $\mathcal{S}=\Delta^{K-1}$ for $K$-class classification. 
A \emph{prediction set} procedure maps $(X,A)$ to a subset $\mathcal{C}(X,A)\subseteq\mathcal{Y}$.

For causal notation, we adopt structural causal models (SCMs)~\citep{pearl2009causality}.
An SCM is a tuple $(U,V,F)$, where $U$ are exogenous variables, $V=\{X,A,Y\}$ are endogenous variables, and $F$ are structural equations that determine each variable in $V$ as a function of its parents and some exogenous noise.
Counterfactuals consider the following question: \textit{``What would the value of observed variables be, had we intervened to set $A$ to a different value?''}
We denote by $X_{A\leftarrow a}$ the counterfactual version of $X$ under the intervention $do(A=a)$, and similarly for $Y_{A\leftarrow a}$.
For more details on SCMs and counterfactuals, see \citet{pearl2009causality}.

\subsection{Split Conformal Prediction}
\label{subsec:split-cp}
Here we briefly describe conformal prediction (CP); see \citet{angelopoulos2023conformal} for a recent survey.
Conformal prediction (CP) constructs prediction sets with a pre-specified probability $1-\alpha$ of containing the true label, without distributional assumptions beyond exchangeability \citep{vovk2005algorithmic, shafer2008tutorial}.
For simplicity, we focus on the popular \emph{split} (or inductive) CP variant \citep{papadopoulos2002inductive}. 
The protected attribute $A$ is not needed for standard CP, but we include it in the notation for clarity and to make explicit the dependence used in later sections.
Suppose we are given a pre-trained predictor $\hat{f}$ (e.g., a neural network) and a calibration dataset of size $n$, $\{(X_i,A_i,Y_i)\}_{i=1}^n$, independent of $\hat{f}$.
The goal is to construct a prediction set $\mathcal{C}_{\alpha}(X,A)$ that covers the true label $Y$ with probability of at least $1-\alpha$ over the randomness of $(X,A,Y)$ and the calibration dataset. 
To this end, we define a \emph{conformity score} function $s:\mathcal{X}\times\mathcal{A}\times\mathcal{Y}\to\mathbb{R}$ that measures how well a candidate label $y$ conforms with the features $(x,a)$ under the predictor $\hat{f}$.
Without loss of generality, we use a negative-oriented score, where smaller values indicate better conformity.
The score function inherently depends on the predictor $\hat{f}$ but for ease of notation we omit this dependence.
Then, one computes conformity scores on the calibration set:
\begin{equation}
    S_i = s\left(X_i,A_i,Y_i\right), \qquad i\in\mathcal{I}_{\mathrm{cal}},
\end{equation}
where $\mathcal{I}_{\mathrm{cal}}$ denotes the index set of the calibration data.
Let $\hat{q}_{1-\alpha}$ be the $(1-\alpha)(1+1/|\mathcal{I}_{\mathrm{cal}}|)$-empirical quantile of $\{S_i\}_{i\in\mathcal{I}_{\mathrm{cal}}}$. 
For a test point $(X_{n+1},A_{n+1})$, the split conformal prediction set is computed as follows:
\begin{equation}
    \mathcal{C}_{\alpha}(X_{n+1},A_{n+1}) = \left\{y\in\mathcal{Y}: s(X_{n+1},A_{n+1},y) \le \hat{q}_{1-\alpha}\right\}
\end{equation}
If $(X_i,A_i,Y_i)$ and $(X_{n+1},A_{n+1},Y_{n+1})$ are exchangeable, then $\mathcal{C}_{\alpha}(X_{n+1},A_{n+1})$ satisfies the \emph{marginal coverage} guarantee:
\begin{equation}
    \Pr\big\{Y_{n+1} \in \mathcal{C}_{\alpha}(X_{n+1},A_{n+1})\big\} \ge 1-\alpha.
\end{equation}
If the score values are almost surely distinct, the probability above is upper bounded by $1-\alpha + 1/(n+1)$.
A detailed proof of this result can be found in \citet{lei2018distribution}.
The coverage guarantee is marginal, and the randomness is over both the calibration and test points.

\section{COUNTERFACTUAL FAIRNESS}
\label{sec:problem}

In this section, we formalize counterfactual fairness for point predictors and extend this definition to prediction set procedures.
We also briefly recall past works that ensure counterfactual fairness for point predictors.

\subsection{Point predictors}
 Building on the causal framework of \citet{pearl2009causality}, \citet{kusner2017counterfactual} introduced the notion of counterfactual fairness (CF) which requires that a predictor's output remains invariant under counterfactual interventions on protected attributes. More formally:
\begin{definition}[Counterfactual Fairness---point predictor]
\label{def:cf}
A predictor $\hat{f}:\mathcal{X}\times\mathcal{A}\to\mathcal{S}$ is \emph{counterfactually fair} (CF) if, for all $a,a'\in\mathcal{A}$ and $s\in\mathcal{S}$,
\begin{multline}
\label{eq:cf-prob-def}
    \Pr\left\{ \hat{f}(X_{A\leftarrow a},a) = s \mid A=a, U=u \right\} =\\
    \Pr\left\{ \hat{f}(X_{A\leftarrow a'},a') = s \mid A=a, U=u \right\}
\end{multline}
\end{definition}
That is, after intervening on $A$ while holding the latent $U$ fixed, the distribution of predictions should not change. 
Equivalently, $\hat f(X_{A\leftarrow a},a)\mid(A=a,U=u)$ and $\hat f(X_{A\leftarrow a'},a')\mid(A=a,U=u)$ have the same conditional distribution.

Several strategies exist to design CF predictors $\hat{f}$ satisfying \eqref{eq:cf-prob-def}.
\citet{kusner2017counterfactual} propose to learn a predictor on only on the features that are not descendants of $A$ in the causal graph, then no causal path exists from $A$ to $\hat{f}(X,A)$, ensuring CF.
This approach may be impractical if all features are descendants of $A$.
Counterfactually Fair Representations (CFR) \citep{zuo2023counterfactually} satisfies fairness by learning a latent representation $Z$ for factual and counterfactual instances, then combines them via a permutation-invariant operation (e.g., mean) before feeding to a predictor.
\citet{zhou2024counterfactual} propose a method that satisfies CF by combining both factual and counterfactual predictions weighted by the probability of the corresponding protected attribute value, called Plug-in Counterfactual Fairness (PCF).

\subsection{Prediction sets}
We extend counterfactual fairness notion from \emph{point predictors} to \emph{prediction sets} by requiring that the set produced for an individual remain invariant under counterfactual interventions on the protected attribute. 
We call this property \emph{set-level counterfactual fairness}.
\begin{definition}[Counterfactual fairness---prediction sets]
\label{def:cfcp}
A set-valued procedure $\mathcal{C}_\alpha$ is \emph{counterfactually fair}
if, for all $a,a'\in\mathcal{A}$ and $y\in\mathcal{Y}$,
\begin{multline}
\label{eq:cf-set-def}
    \Pr\left\{ y \in \mathcal{C}_\alpha(X_{A\leftarrow a},a) \mid A=a, U=u \right\} = \\
    \Pr\left\{ y \in \mathcal{C}_\alpha(X_{A\leftarrow a'},a') \mid A=a, U=u \right\}.
\end{multline}
\end{definition}
That is, for any individual with latent attributes $U=u$ and protected attribute $A=a$, the probability of including any label $y$ in the prediction set is invariant to counterfactual changes in $A$.

\section{COUNTERFACTUALLY FAIR CONFORMAL PREDICTION}
\label{sec:method}

We aim to design a split-CP procedure that preserves marginal coverage while enforcing set-level counterfactual invariance. Precisely, given a pre-trained (possibly counterfactually unfair) predictor $\hat f$, a calibration set $\mathcal{I}_{\mathrm{cal}}$, and a target level $1-\alpha$,
construct a conformal prediction procedure $\mathcal{C}_\alpha$ that (i) preserves the standard marginal coverage guarantee and (ii) satisfies set-level counterfactual fairness (Def.~\ref{def:cfcp}).

In this section, we propose an algorithm that satisfies these two properties. 
Then we compare CF-CP to baselines and discuss the differences.
Finally, we demonstrate that this procedure results in counterfactually fair prediction sets along with marginal coverage guarantees.

\subsection{Score Symmetrization and CF-CP Sets}
\label{subsec:method-intro}
Let $s:\mathcal{X}\times\mathcal{A}\times\mathcal{Y}\to\mathbb{R}$ be any conformity score used by split CP with a base predictor $\hat f$.
We define the \emph{counterfactual score symmetrization} by aggregating scores over protected-attribute interventions:
\begin{equation}
\label{eq:s-sym}
    s_{\mathrm{cf}}(x,a,y) = \mathrm{Agg}\left\{s\left(x_{A\leftarrow a'},a',y\right) : a'\in\mathcal{A}\right\}
\end{equation}
where $\mathrm{Agg}$ is a symmetric operator that is invariant to permutations of its inputs and $x_{A\leftarrow a'}$ denotes the counterfactual features obtained by setting $A{:=}a'$. 
Examples of the operator $\mathrm{Agg}$ include mean, max and min operations---we elaborate more on this below.

We then run split CP \emph{with} $s_{\mathrm{cf}}$: on calibration points $\{(X_i,A_i,Y_i)\}_{i\in\mathcal{I}_{\mathrm{cal}}}$ compute $S^{\mathrm{cf}}_i = s_{\mathrm{cf}}(X_i,A_i,Y_i)$, take the usual quantile $\hat q_{1-\alpha}$ with finite sample correction.
At test time, for a new point $(x,a)$, the CF-CP set is
\begin{equation}
\label{eq:cfcp-set}
    \mathcal{C}_\alpha^{\mathrm{CF}}(x,a) = \left\{y\in\mathcal{Y} : s_{\mathrm{cf}}(x,a,y) \le \hat q_{1-\alpha}\right\}.
\end{equation}
This procedure is summarized in Algorithm~\ref{alg:cfcp}.

\begin{algorithm}[t]
\caption{Split CF-CP via score symmetrization}
\label{alg:cfcp}
\begin{algorithmic}[1]
    \Require Base predictor $\hat f$, score $s(\cdot)$, calibration set $\{(X_i,A_i,Y_i)\}_{i\in\mathcal{I}_{\mathrm{cal}}}$, target $1-\alpha$, aggregator $\mathrm{Agg}$
    \State Define \begin{equation*}s_{\mathrm{cf}}(x,a,y) \leftarrow \mathrm{Agg}\left\{s\left(x_{A\leftarrow a'}, a', y\right) : a'\in\mathcal{A}\right\}\end{equation*}
    \State \textbf{For each} $i\in\mathcal{I}_{\mathrm{cal}}$: compute $S^{\mathrm{cf}}_i \leftarrow s_{\mathrm{cf}}(X_i,A_i,Y_i)$
    \State Set $\hat q_{1-\alpha} \leftarrow \mathrm{Quantile}_{\lceil (|\mathcal{I}_{\mathrm{cal}}|+1)(1-\alpha)\rceil/|\mathcal{I}_{\mathrm{cal}}|}\left(\{S^{\mathrm{cf}}_i\}\right)$
    \State \textbf{At test time} for $(x,a)$, form
    \begin{equation*}\mathcal{C}_\alpha^{\mathrm{CF}}(x,a) \leftarrow \left\{ y: s_{\mathrm{cf}}(x,a,y) \le \hat q_{1-\alpha}\right\}\end{equation*}
    \Ensure Prediction set $\mathcal{C}_\alpha^{\mathrm{CF}}(x,a)$
\end{algorithmic}
\end{algorithm}

\paragraph{Aggregator choices.} The only requirement on $\mathrm{Agg}$ is to be permutation-invariant--i.e., the output of the aggregator should not change with the order of the inputs.
In this paper, we consider three simple choices: (i) \texttt{mean}, (ii) \texttt{max}, and (iii) \texttt{min}.
During the calibration threshold learning (step 3 of Algorithm~\ref{alg:cfcp}), \texttt{max} leads to larger thresholds $\hat q_{1-\alpha}$, because it takes the maximum score across interventions for each sample in calibration dataset.
However, at test time (step 4 of Algorithm~\ref{alg:cfcp}), the \texttt{max} aggregator leads to intersection of the per-intervention sets, which balances out the effect of a larger threshold.
Conversely, \texttt{min} leads to smaller thresholds at the calibration phase, but the test-time sets are unions of per-intervention sets, which can lead to larger prediction sets.
The \texttt{mean} aggregator is a compromise between these cases in both calibration and test phases.
While we focus on these simple aggregators, more complex choices are possible as long as they are permutation-invariant.

\paragraph{Computational overhead.}
Compared to standard split-CP, CF-CP introduces two sources of additional overhead: (i) generating counterfactual interventions, and (ii) computing conformity scores for each intervention $a \in \mathcal{A}$.
Overall, CF-CP incurs an additional $\mathcal{O}(|\mathcal{A}|)$ factor compared to split-CP.
Thus, when the protected attribute has small cardinality, the computational cost of CF-CP remains close to that of split-CP.

\subsection{Baselines and Comparisons}
We compare CF-CP to two types of baselines that either retrain a fair point predictor or apply post-hoc set manipulations, and we note their limitations.

\paragraph{Retraining-based approaches.}
One option is to first obtain a point predictor that is counterfactually fair—e.g., by excluding descendants of $A$ in the SCM or via counterfactually fair representations (CFR) / plug-in counterfactual fairness (PCF)—and then apply standard split-CP on its score. 
This can yield set-level invariance by construction, but it typically trades predictive accuracy for fairness and thus inflates set sizes to maintain coverage; it also requires retraining, which may be infeasible. 
In addition, PCF variant relies on access to the probability distribution of the protected attribute, $\Pr(A)$, for post-training adjustments.

\paragraph{Post-hoc union over interventions.}
A second baseline constructs separate CP sets under each protected-attribute intervention and unions them:
\begin{equation}
    \mathcal{C}^{\cup}_\alpha(x):=\bigcup_{a'\in\mathcal{A}} \mathcal{C}_\alpha(x_{A\leftarrow a'}, a')
\end{equation}
This is again counterfactually fair by construction, but the union systematically enlarges sets and pushes effective coverage above the nominal $1-\alpha$.

\paragraph{How CF-CP differs.}
CF-CP calibrates a \emph{single, symmetrized} conformity score that aggregates per-intervention scores at the instance level. 
In contrast to the post-hoc union, CF-CP avoids the union-induced overcoverage while still ensuring set-level counterfactual invariance.
Compared to retraining-based approaches (CFR/PCF or descendants-exclusion), CF-CP is training-free, does not require modifying features or imposing invariance penalties, and does not need access to $\Pr(A)$; all aggregation occurs on scores, not on model parameters or population distributions.

\subsection{Theoretical Guarantees}
\label{subsec:theory}

In this section we show that CF-CP satisfies the set-level counterfactual fairness of Definition~\ref{def:cfcp} while preserving the marginal coverage guarantee of split CP. First, we state the main assumption needed for our theoretical analysis.

\begin{assumption}
    The structural mapping between $X$ and $U$ is invertible for given $A$. Also, $U$ and $A$ are independent. 
    \label{assump:invertible}
\end{assumption}

The independence assumption is standard in the counterfactual fairness literature.
Even though the invertibility assumption may seem strong, it is commonly used in the counterfactual estimation and fairness literature \citep{nasr2023counterfactual,zhou2024towards,zhou2024counterfactual}.
We also note that after relaxing the assumption, one can still achieve approximate counterfactual fairness, as we show empirically in Section~\ref{sec:experiments}.

Assumption~\ref{assump:invertible} asserts that there exists a mapping $g:\mathcal{U}\times\mathcal{A}\to\mathcal{X}$ such that $X = g(U,A)$ and, for each $a\in\mathcal{A}$, the mapping $g_a(u):=g(u,a)$ is invertible with inverse $g_a^{-1}:\mathcal{X}\to\mathcal{U}$.
For any target $a'\in\mathcal{A}$, the corresponding counterfactual features are therefore
\begin{equation}
    x_{A\leftarrow a'} = g_{a'}(U) = g_{a'}\left(g_a^{-1}(x)\right).
\end{equation}
Also, the symmetrized score can be written as:
\begin{equation}
    s_{\mathrm{cf}}(x,a,y)
    =
    \mathrm{Agg}\left\{s\left(g_{a'}(g_a^{-1}(x)),a',y\right) : a'\in\mathcal{A}\right\}
\end{equation}
This invertibility makes precise that the multiset $\{x_{A\leftarrow a'}\}_{a'}$ depends only on the shared latent $U$, not on the factual $a$.
This results in invariance of the symmetrized score.
We formalize this observation in the following lemma.

\begin{lemma}[Invariance of the symmetrized score]
\label{lem:score-invariance}
Let Assumption~\ref{assump:invertible} hold.
Fix any $y\in\mathcal{Y}$.
Then for all $a,a'\in\mathcal{A}$,
\begin{equation}
\label{eq:sym-score-invariance}
    s_{\mathrm{cf}}\left(X_{A\leftarrow a},a,y\right) = s_{\mathrm{cf}}\left(X_{A\leftarrow a'},a',y\right).
\end{equation}
\end{lemma}

The proof is given in the Appendix~\ref{sec:proof-lem-score-invariance}.
The key idea is that, under Assumption~\ref{assump:invertible}, the latent $U$ can be uniquely recovered from the factual features $X$ and protected attribute $A=a$ via the inverse function $g_a^{-1}$.
This allows us to express all counterfactual features $X_{A\leftarrow a'}$ as functions of $X_{A\leftarrow a}$ and $A=a$, and vice versa.
The symmetry of $s_{\mathrm{cf}}$ then follows from the permutation-invariance of $\mathrm{Agg}$.

This score-level invariance leads to the desired set-level counterfactual fairness.
Since the score value is the same across counterfactuals, the prediction sets formed by thresholding the score at $\hat q_{1-\alpha}$ must also coincide.

\begin{corollary}[Set-level counterfactual invariance]
\label{cor:set-invariance}
Under the conditions of Lemma~\ref{lem:score-invariance}, for all $a,a'\in\mathcal{A}$,
\begin{equation}
\label{eq:set-invariance}
    \mathcal{C}_\alpha^{\mathrm{CF}}\big(X_{A\leftarrow a},a\big) = \mathcal{C}_\alpha^{\mathrm{CF}}\big(X_{A\leftarrow a'},a'\big).
\end{equation}
\end{corollary}
The proof is immediate from Lemma~\ref{lem:score-invariance} and the definition of CF-CP sets in~\eqref{eq:cfcp-set}.

The previous lemma and corollary establish that CF-CP enforces \emph{set-level counterfactual invariance}.
We now argue that replacing $s$ with $s_{\mathrm{cf}}$ does \emph{not} disturb the exchangeability required by split CP, so the standard marginal coverage guarantee continues to hold.

\begin{theorem}[Exchangeability preservation]
\label{thm:exch}
Assume that Assumption~\ref{assump:invertible} holds and that the calibration points $\{(X_i,A_i,Y_i)\}_{i\in\mathcal{I}_{\mathrm{cal}}}$ and the test point $(X_{n+1},A_{n+1},Y_{n+1})$ are exchangeable.
Suppose the score function $s$ and the aggregator $\mathrm{Agg}$ are measurable, and $\mathrm{Agg}$ is permutation-invariant.
Define the counterfactual symmetrization
\begin{equation*}
    s_{\mathrm{cf}}(x,a,y) = \mathrm{Agg}\left\{ s\left(x_{A\leftarrow a'},a',y\right) : a'\in\mathcal{A}\right\}.
\end{equation*}
Then, the scores
\begin{equation}
    \{S^{\mathrm{cf}}_i = s_{\mathrm{cf}}(X_i,A_i,Y_i) : i \in \mathcal{I}_{\mathrm{cal}} \cup \{n+1\}\}
\end{equation}
are exchangeable.
\end{theorem}
The proof is given in the Appendix~\ref{sec:proof-thm-exch}.
The main idea is that $s_{\mathrm{cf}}$ is a coordinate-wise measurable transformation of the exchangeable tuples $(X_i,A_i,Y_i)$, hence the resulting scores remain exchangeable.

Theorem~\ref{thm:exch} shows that the exchangeability structure required by split CP is preserved when using the symmetrized score $s_{\mathrm{cf}}$.
This leads to the following coverage guarantee for CF-CP.

\begin{corollary}[Marginal coverage for CF-CP]
\label{cor:coverage}
Under the conditions of Theorem~\ref{thm:exch}, the split conformal set constructed with $s_{\mathrm{cf}}$ satisfies the standard marginal coverage guarantee:
\begin{equation}
    \Pr\big\{Y_{n+1} \in \mathcal{C}_\alpha^{\mathrm{CF}}(X_{n+1},A_{n+1})\big\} \ge 1-\alpha.
\end{equation}
\end{corollary}
The proof directly follows the standard split CP coverage proof \citep{vovk2005algorithmic,lei2018distribution} applied to the exchangeable scores $S^{\mathrm{cf}}_i$.

Combining Corollaries~\ref{cor:set-invariance} and \ref{cor:coverage}, we conclude that CF-CP achieves both marginal coverage and set-level counterfactual invariance, as desired.

\section{EXPERIMENTS}
\label{sec:experiments}

In this section, we empirically evaluate CF-CP on synthetic and real datasets, spanning regression and classification tasks and show that it achieves the desired counterfactual fairness while meeting the target coverage rate with minimal increase in the size of the prediction sets is available in the following repository.\footnote{\url{https://github.com/guldoganozgur/cf_cp}}

\paragraph{Methods.}
We compare the following methods: (i) \textbf{Split CP}: standard split conformal prediction using the original conformity score with the base predictor; (ii) \textbf{Post-hoc Union}: computes CP sets under each intervention and takes their union, specifically, $\mathcal{C}^{\cup}_\alpha(x):=\bigcup_{a'\in\mathcal{A}} \mathcal{C}_\alpha(x_{A\leftarrow a'}, a')$; (iii) \textbf{Counterfactual Fairness with $U$}~\citep{kusner2017counterfactual}: train a counterfactually fair predictor by excluding descendants of $A$ in the SCM, to simplify this we train base model only on $U$ similar to the approach in \citet{zhou2024counterfactual} followed by split CP; (iv) \textbf{Counterfactual Fair Representations (CFR)}~\citep{zuo2023counterfactually}: trains a predictor on the augmented features as the mean of factual and counterfactual features, specifically, $\tilde{X} = \frac{1}{|\mathcal{A}|}\sum_{a'\in\mathcal{A}} X_{A\leftarrow a'}$ and then applies split CP; (v) \textbf{Plug-in Counterfactual Fairness (PCF)}~\citep{zhou2024counterfactual}: uses a weighted combination of the predictions for each intervention, specifically, $\hat{f}(X,A) = \sum_{a'\in\mathcal{A}} \Pr(A=a') \hat{f}(X_{A\leftarrow a'},a')$ where $\hat{f}$ is the base predictor and then follows with split CP; and lastly our method (vi) \textbf{CF-CP}: our proposed method that uses the counterfactual symmetrized score in split CP, and we consider three aggregation functions: \texttt{mean}, \texttt{max}, and \texttt{min}.
Following common practice, in the classification tasks if the prediction set is empty, we include the label with the highest predicted probability.
The base predictor $\hat{f}$ is a linear model for both regression and classification tasks.
More details on hyperparameters and training are in the Appendix~\ref{sec:implementation-details}.

\paragraph{Score functions.}
For regression tasks, we use the absolute residual score $s(x,a,y) = |\hat{f}(x,a) - y|$.
So, the less the absolute error, the better the conformity.
For classification tasks, we use Least Ambiguous Classifier (LAC)~\citep{sadinle2019least} score $s(x,a,y) = 1 - \hat{f}_y(x,a)$, where $\hat{f}_y(x,a)$ is the predicted probability of class $y$.
The more probable the class, the better the conformity.
We also provide additional results with different base score functions in Appendix~\ref{sec:additional-results}. 

\begin{table*}[ht]
\centering
\caption{Synthetic results at $\alpha=0.1$ (mean $\pm$ std over 10 runs). Left block: regression; right block: classification. ($\downarrow$) and ($\uparrow$) indicates that a smaller and larger value is desirable, respectively. Highlighting follows the convention described in the evaluation metrics paragraph.}
\label{tab:synth-combined}
\setlength{\tabcolsep}{4.5pt}
\renewcommand{\arraystretch}{1.05}
\resizebox{\textwidth}{!}{
\begin{tabular}{lccccc|ccccc}
\toprule
& \multicolumn{5}{c}{Regression} & \multicolumn{5}{c}{Classification} \\
\cmidrule(lr){2-6}\cmidrule(lr){7-11}
Method & MSE ($\downarrow$) & TE ($\downarrow$) & Coverage & Avg. Set Size ($\downarrow$) & CSD ($\downarrow$) & Accuracy ($\uparrow$) & TE ($\downarrow$) & Coverage & Avg. Set Size ($\downarrow$) & CSD ($\downarrow$) \\
\midrule
SplitCP & 0.377 $\pm$ 0.012 & 1.219 $\pm$ 0.014 & 0.901 $\pm$ 0.008 & 2.032 $\pm$ 0.038 & 0.713 $\pm$ 0.009 & 0.730 $\pm$ 0.011 & 0.543 $\pm$ 0.037 & 0.905 $\pm$ 0.011 & 2.025 $\pm$ 0.209 & 0.642 $\pm$ 0.043 \\
Post-hoc Union & 0.377 $\pm$ 0.012 & 1.219 $\pm$ 0.014 & 0.944 $\pm$ 0.005 & 3.251 $\pm$ 0.043 & 0.000 $\pm$ 0.000 & 0.730 $\pm$ 0.011 & 0.543 $\pm$ 0.037 & 0.936 $\pm$ 0.010 & 3.123 $\pm$ 0.255 & 0.000 $\pm$ 0.000 \\
CFU & 1.112 $\pm$ 0.017 & 0.000 $\pm$ 0.000 & 0.901 $\pm$ 0.008 & 3.447 $\pm$ 0.088 & 0.000 $\pm$ 0.000 & \dunderline{0.589 $\pm$ 0.017} & 0.000 $\pm$ 0.000 & 0.907 $\pm$ 0.009 & 2.700 $\pm$ 0.125 & 0.000 $\pm$ 0.000 \\
CFR & 0.832 $\pm$ 0.014 & 0.000 $\pm$ 0.000 & 0.899 $\pm$ 0.010 & 2.964 $\pm$ 0.055 & 0.000 $\pm$ 0.000 & \dunderline{0.589 $\pm$ 0.017} & 0.000 $\pm$ 0.000 & 0.908 $\pm$ 0.009 & 2.700 $\pm$ 0.125 & 0.000 $\pm$ 0.000 \\
PCF & \dunderline{0.826 $\pm$ 0.013} & 0.000 $\pm$ 0.000 & 0.898 $\pm$ 0.008 & \dunderline{2.951 $\pm$ 0.046} & 0.000 $\pm$ 0.000 & 0.571 $\pm$ 0.017 & 0.000 $\pm$ 0.000 & 0.905 $\pm$ 0.012 & \bftab{2.526 $\boldsymbol\pm$ 0.166} & 0.000 $\pm$ 0.000 \\
\midrule
CF-CP-mean & \bftab{0.377 $\boldsymbol\pm$ 0.012} & 1.219 $\pm$ 0.014 & 0.899 $\pm$ 0.010 & 2.971 $\pm$ 0.050 & 0.000 $\pm$ 0.000 & \bftab{0.730 $\boldsymbol\pm$ 0.011} & 0.543 $\pm$ 0.037 & 0.904 $\pm$ 0.014 & \dunderline{2.542 $\pm$ 0.185} & 0.000 $\pm$ 0.000 \\
CF-CP-max & \bftab{0.377 $\boldsymbol\pm$ 0.012} & 1.219 $\pm$ 0.014 & 0.900 $\pm$ 0.009 & 3.275 $\pm$ 0.053 & 0.000 $\pm$ 0.000 & \bftab{0.730 $\boldsymbol\pm$ 0.011} & 0.543 $\pm$ 0.037 & 0.935 $\pm$ 0.006 & 3.696 $\pm$ 0.168 & 0.038 $\pm$ 0.008 \\
CF-CP-min & \bftab{0.377 $\boldsymbol\pm$ 0.012} & 1.219 $\pm$ 0.014 & 0.900 $\pm$ 0.010 & \bftab{2.897 $\boldsymbol\pm$ 0.058} & 0.000 $\pm$ 0.000 & \bftab{0.730 $\boldsymbol\pm$ 0.011} & 0.543 $\pm$ 0.037 & 0.906 $\pm$ 0.013 & 2.581 $\pm$ 0.180 & 0.000 $\pm$ 0.000 \\
\bottomrule
\end{tabular}

}
\end{table*}

\paragraph{Evaluation metrics.}
For each dataset and method, we report the following metrics---$\uparrow$ and $\downarrow$ indicate that a larger and smaller value is desirable, respectively: 
(i) \textbf{MSE ($\downarrow$)/Accuracy ($\uparrow$)}: the mean squared error for regression and accuracy for classification on the test set;
(ii) \textbf{Total Effect(TE) ($\downarrow$)}: measures the counterfactual fairness of the base predictor $\hat{f}$, for each test point, we compute the change in the model's output when $A$ is counterfactually flipped, and then average over the test set;
(iii) \textbf{Coverage}: the empirical coverage of the prediction sets on the test set, $\frac{1}{|\mathcal{I}_{\mathrm{test}}|} \sum_{i\in\mathcal{I}_{\mathrm{test}}} \mathbb{I}[Y_i \in \mathcal{C}_\alpha(X_i,A_i)]$; 
(iv) \textbf{Avg. Set Size ($\downarrow$)}: the average size of the prediction sets on the test set, for regression it is the average length of the intervals, and for classification it is the average cardinality of the sets, specifically, $\frac{1}{|\mathcal{I}_{\mathrm{test}}|} \sum_{i\in\mathcal{I}_{\mathrm{test}}} |\mathcal{C}_\alpha(X_i,A_i)|$; and 
(v) \textbf{Counterfactual Set Disparity (CSD) ($\downarrow$)}: quantifies how much the set changes across protected-attribute interventions for each test point.
The details of CSD are as follows:
For binary $A$, we use the Jaccard distance between the factual and counterfactual sets:
\begin{equation}
\label{eq:csd}
    \mathrm{CSD}(x) = 1 - \frac{|\mathcal{C}_\alpha(x_{A\leftarrow a},a) \cap \mathcal{C}_\alpha(x_{A\leftarrow a'},a')|}{|\mathcal{C}_\alpha(x_{A\leftarrow a},a) \cup \mathcal{C}_\alpha(x_{A\leftarrow a'},a')|}.
\end{equation}
For multi-valued $A$, it can be extended by averaging pairwise Jaccard distances across interventions relative to the factual value.
A lower CSD indicates more similar sets across counterfactuals, with 0 indicating identical sets and 1 indicating non-overlapping sets.
We report the average counterfactual set disparity on the test set, specifically, $\frac{1}{|\mathcal{I}_{\mathrm{test}}|} \sum_{i\in\mathcal{I}_{\mathrm{test}}} \mathrm{CSD}(X_i)$.
The metrics for each method are computed over 10 random splits of the data into training, calibration, and test sets, and we report the mean and standard deviation in the tables and figures.
For all tables in this paper, we use a common highlighting convention.
For the tables that use oracle counterfactuals, bold (underlined) values mark the best (second-best) MSE/Accuracy and Avg.\ Set Size among the methods that achieve the target coverage and perfect counterfactual fairness for prediction sets (CSD=0).
For the tables with noisy counterfactuals, bold (underlined) values mark the best (second-best) CSD, since in this setting we are primarily interested in how fairness behaves under imperfect counterfactuals.

\begin{table*}[ht]
\centering
\caption{Real data results. Left block: regression on Law School dataset with $\alpha=0.1$; right block: classification on Bios dataset with $\alpha=0.025$ (mean $\pm$ std over 10 runs). Highlighting follows the convention described in the evaluation metrics paragraph.}
\label{tab:real-combined}
\setlength{\tabcolsep}{4.5pt}
\renewcommand{\arraystretch}{1.05}
\resizebox{\textwidth}{!}{
\begin{tabular}{lccccc|ccccc}
\toprule
& \multicolumn{5}{c}{Law School Regression} & \multicolumn{5}{c}{Bios Classification} \\
\cmidrule(lr){2-6}\cmidrule(lr){7-11}
Method &  MSE ($\downarrow$) & TE ($\downarrow$) & Coverage & Avg. Set Size ($\downarrow$) & CSD ($\downarrow$) & Accuracy ($\uparrow$) & TE ($\downarrow$) & Coverage & Avg. Set Size ($\downarrow$) & CSD ($\downarrow$) \\
\midrule
SplitCP & 0.758 $\pm$ 0.005 & 0.723 $\pm$ 0.019 & 0.898 $\pm$ 0.009 & 2.847 $\pm$ 0.070 & 0.405 $\pm$ 0.010 & 0.801 $\pm$ 0.001 & 0.094 $\pm$ 0.003 & 0.975 $\pm$ 0.003 & 3.311 $\pm$ 0.199 & 0.199 $\pm$ 0.005 \\
Post-hoc Union & 0.758 $\pm$ 0.005 & 0.723 $\pm$ 0.019 & 0.944 $\pm$ 0.007 & 3.570 $\pm$ 0.078 & 0.000 $\pm$ 0.000 & 0.801 $\pm$ 0.001 & 0.094 $\pm$ 0.002 & 0.978 $\pm$ 0.002 & 3.884 $\pm$ 0.227 & 0.000 $\pm$ 0.000 \\
CFU & 0.829 $\pm$ 0.007 & 0.000 $\pm$ 0.000 & 0.898 $\pm$ 0.011 & 2.968 $\pm$ 0.078 & 0.000 $\pm$ 0.000 & - & - & - & - & - \\
CFR & \dunderline{0.827 $\pm$ 0.007} & 0.000 $\pm$ 0.000 & 0.897 $\pm$ 0.012 & \dunderline{2.967 $\pm$ 0.084} & 0.000 $\pm$ 0.000 & \dunderline{0.799 $\pm$ 0.000} & 0.000 $\pm$ 0.000 & 0.974 $\pm$ 0.003 & \bftab{3.379 $\boldsymbol\pm$ 0.183} & 0.000 $\pm$ 0.000 \\
PCF & 0.828 $\pm$ 0.007 & 0.000 $\pm$ 0.000 & 0.897 $\pm$ 0.011 & \bftab{2.958 $\boldsymbol\pm$ 0.072} & 0.000 $\pm$ 0.000 & 0.795 $\pm$ 0.001 & 0.000 $\pm$ 0.000 & 0.974 $\pm$ 0.003 & \dunderline{3.471 $\pm$ 0.208} & 0.000 $\pm$ 0.000 \\
\midrule
CF-CP-mean & \bftab{0.758 $\boldsymbol\pm$ 0.005} & 0.723 $\pm$ 0.019 & 0.900 $\pm$ 0.013 & 3.106 $\pm$ 0.100 & 0.000 $\pm$ 0.000 & \bftab{0.801 $\boldsymbol\pm$ 0.001} & 0.093 $\pm$ 0.003 & 0.974 $\pm$ 0.003 & 3.486 $\pm$ 0.191 & 0.000 $\pm$ 0.000 \\
CF-CP-max & \bftab{0.758 $\boldsymbol\pm$ 0.005} & 0.723 $\pm$ 0.019 & 0.900 $\pm$ 0.013 & 3.106 $\pm$ 0.100 & 0.000 $\pm$ 0.000 & \bftab{0.801 $\boldsymbol\pm$ 0.001} & 0.094 $\pm$ 0.003 & 0.975 $\pm$ 0.003 & 3.700 $\pm$ 0.208 & 0.000 $\pm$ 0.000 \\
CF-CP-min & \bftab{0.758 $\boldsymbol\pm$ 0.005} & 0.723 $\pm$ 0.019 & 0.900 $\pm$ 0.013 & 3.106 $\pm$ 0.100 & 0.000 $\pm$ 0.000 & \bftab{0.801 $\boldsymbol\pm$ 0.001} & 0.093 $\pm$ 0.003 & 0.974 $\pm$ 0.003 & 3.508 $\pm$ 0.199 & 0.000 $\pm$ 0.000 \\
\bottomrule
\end{tabular}

}
\end{table*}

\paragraph{Datasets.}
We evaluate on synthetic and real datasets spanning regression and multiclass classification. 
Throughout our experiments, we assume access to an underlying structural causal model that generates the data; for the synthetic datasets, the SCM is specified by construction, and for the real datasets, the SCMs are obtained by fitting structural equations to the observed data; full details on data generation, preprocessing, SCM specification and fitting, and counterfactual construction are in Appendix~\ref{sec:datasets}.

\textbf{Synthetic Regression.}
A nonlinear SCM with binary protected attribute $A$ generates $(X,Y)$; the task is to predict $Y$ from $(X,A)$ adopted from \citet{zuo2023counterfactually}.
Counterfactuals are formed by flipping $A$ in the $X$-equation while holding exogenous noise fixed. 

\textbf{Synthetic Classification.}
A multi-class ($K{=}10$) SCM extends prior binary settings; the task is to predict $Y$ from $(X,A)$. 
It is inspired by \citet{zhou2024counterfactual} but extended to multi-class targets.
Counterfactuals replace $A$ only in the structural equation for $X$ (latent noise held fixed).

\textbf{Law School.}
Predict first-year average (FYA) from LSAT, UGPA, race, and gender~\citep{wightman1998lsac}; we take race as $A$ (binary: white vs non-white).
Counterfactuals are obtained by fitting a LiNGAM-based SCM to the observed data and intervening on $A$.

\textbf{Bias in Bios.}
Predict profession (28 classes) from biography text; $A$ is gender (binary)~\citep{de2019bias}.
Counterfactuals for this dataset are constructed by the method proposed in \citet{lemberger2024explaining}, which involves first obtaining text embeddings via a pre-trained BERT model~\citep{devlin2019bert}, then decomposing the embedding into components aligned with and independent of $A$; counterfactuals flip the $A$ component while keeping the independent component fixed.

Except for Bios dataset, we use $\alpha=0.1$; for Bios we use $\alpha=0.025$ since $\alpha=0.1$ leads to near-singletons due to high accuracy of the base predictor.

\begin{table*}[ht]
\centering
\caption{Results with noisy estimations of counterfactuals (mean $\pm$ std over 10 runs). Highlighting follows the convention described in the evaluation metrics paragraph.}
\label{tab:noisy-combined}
\setlength{\tabcolsep}{4.5pt}
\renewcommand{\arraystretch}{1.05}
\resizebox{\textwidth}{!}{
\begin{tabular}{lcc|cc|cc|cc}
\toprule
Method & \multicolumn{2}{c}{Synth. Regression} & \multicolumn{2}{c}{Synth. Classification} & \multicolumn{2}{c}{Law School Regression} & \multicolumn{2}{c}{Bios Classification} \\
 & Avg. Set Size ($\downarrow$) & CSD ($\downarrow$) & Avg. Set Size ($\downarrow$) & CSD ($\downarrow$) & Avg. Set Size ($\downarrow$) & CSD ($\downarrow$) & Avg. Set Size ($\downarrow$) & CSD ($\downarrow$) \\
\midrule
SplitCP & 2.032 $\pm$ 0.038 & 0.713 $\pm$ 0.009 & 2.025 $\pm$ 0.209 & 0.642 $\pm$ 0.043 & 2.847 $\pm$ 0.070 & 0.405 $\pm$ 0.010 & 3.311 $\pm$ 0.199 & 0.199 $\pm$ 0.005 \\
Post-hoc Union & 3.259 $\pm$ 0.044 & 0.152 $\pm$ 0.001 & 3.115 $\pm$ 0.254 & 0.260 $\pm$ 0.006 & 3.571 $\pm$ 0.078 & \dunderline{0.036 $\pm$ 0.001} & 3.991 $\pm$ 0.231 & 0.130 $\pm$ 0.002 \\
CFU & 3.551 $\pm$ 0.088 & 0.135 $\pm$ 0.004 & 3.462 $\pm$ 0.239 & 0.345 $\pm$ 0.008 & 2.998 $\pm$ 0.063 & 0.055 $\pm$ 0.002 & - & - \\
CFR & 3.013 $\pm$ 0.078 & \bftab{0.115 $\boldsymbol\pm$ 0.003} & 2.913 $\pm$ 0.209 & \bftab{0.209 $\boldsymbol\pm$ 0.004} & 2.984 $\pm$ 0.064 & \bftab{0.029 $\boldsymbol\pm$ 0.001} & 3.450 $\pm$ 0.226 & \bftab{0.099 $\boldsymbol\pm$ 0.003} \\
PCF & 2.983 $\pm$ 0.061 & \bftab{0.115 $\boldsymbol\pm$ 0.003} & 2.600 $\pm$ 0.210 & 0.263 $\pm$ 0.006 & 2.961 $\pm$ 0.071 & 0.037 $\pm$ 0.001 & 3.560 $\pm$ 0.197 & \dunderline{0.109 $\pm$ 0.002} \\
\midrule
CF-CP-mean & 3.016 $\pm$ 0.054 & \dunderline{0.118 $\pm$ 0.002} & 2.609 $\pm$ 0.203 & 0.259 $\pm$ 0.006 & 3.108 $\pm$ 0.108 & \bftab{0.029 $\boldsymbol\pm$ 0.001} & 3.581 $\pm$ 0.204 & \dunderline{0.109 $\pm$ 0.002} \\
CF-CP-max & 3.374 $\pm$ 0.057 & 0.145 $\pm$ 0.003 & 4.271 $\pm$ 0.185 & \dunderline{0.230 $\pm$ 0.004} & 3.114 $\pm$ 0.112 & 0.041 $\pm$ 0.001 & 3.813 $\pm$ 0.210 & 0.123 $\pm$ 0.002 \\
CF-CP-min & 2.935 $\pm$ 0.050 & 0.168 $\pm$ 0.003 & 2.657 $\pm$ 0.205 & 0.271 $\pm$ 0.005 & 3.107 $\pm$ 0.093 & 0.041 $\pm$ 0.001 & 3.599 $\pm$ 0.165 & 0.128 $\pm$ 0.003 \\
\bottomrule
\end{tabular}
}
\end{table*}

\begin{figure*}[ht]
    \centering
    \begin{subfigure}{0.25\textwidth}
        \centering
        \includegraphics[width=\textwidth]{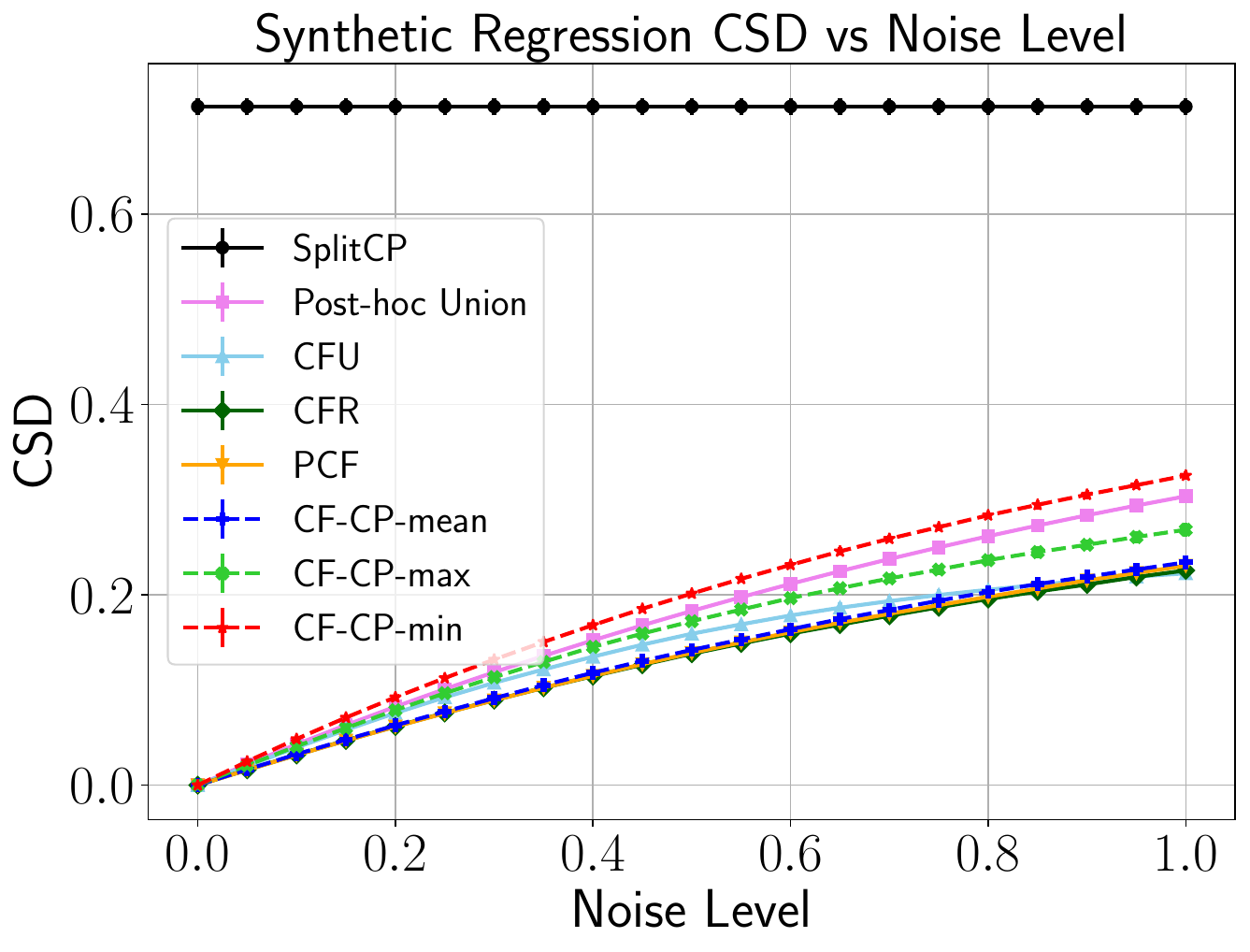}
        \caption{Synthetic Regression}
        \label{fig:synth-reg-csd-noise}
    \end{subfigure}%
    \begin{subfigure}{0.25\textwidth}
        \centering
        \includegraphics[width=\textwidth]{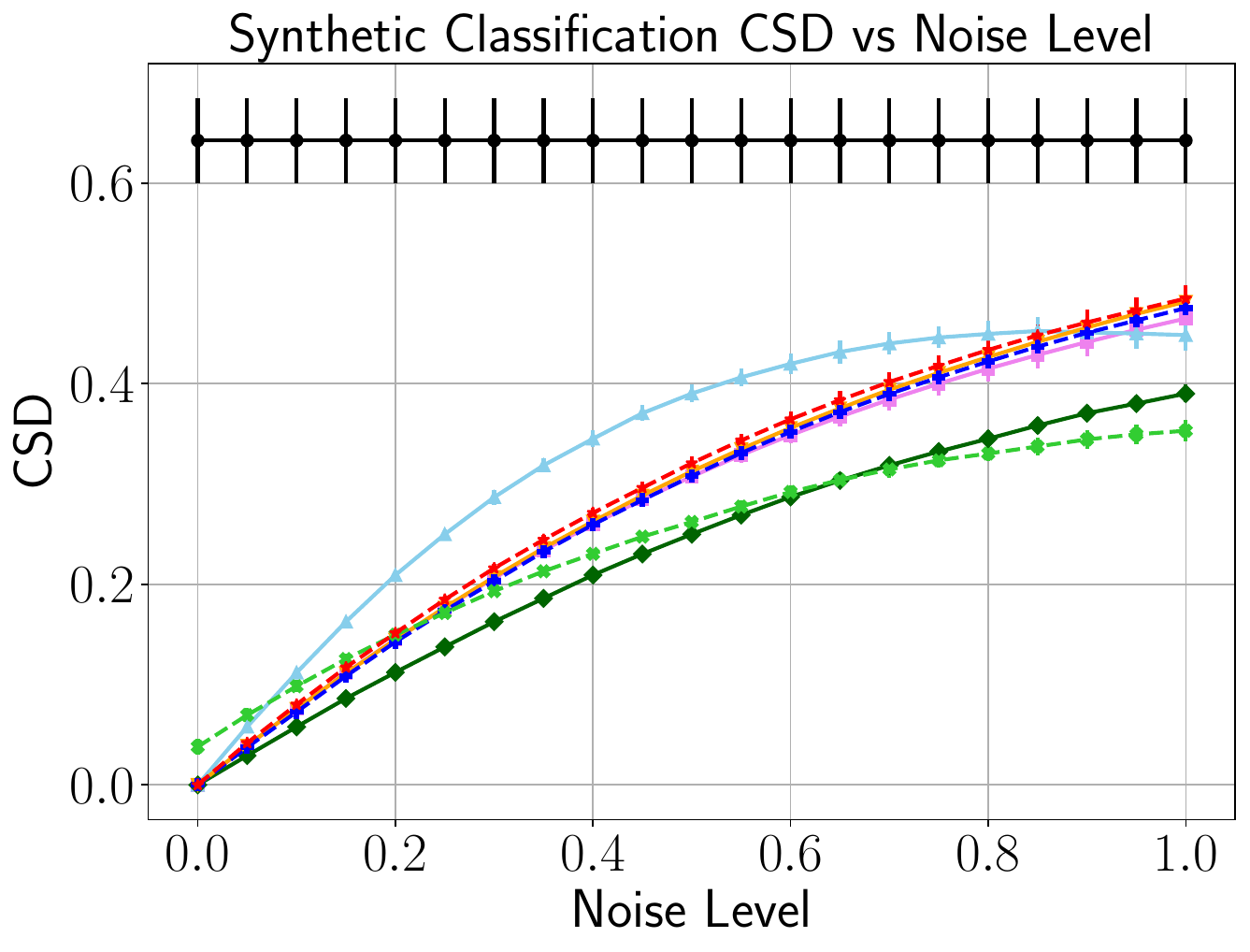}
        \caption{Synthetic Classification}
        \label{fig:synth-clf-csd-noise}
    \end{subfigure}%
    \begin{subfigure}{0.25\textwidth}
        \centering
        \includegraphics[width=\textwidth]{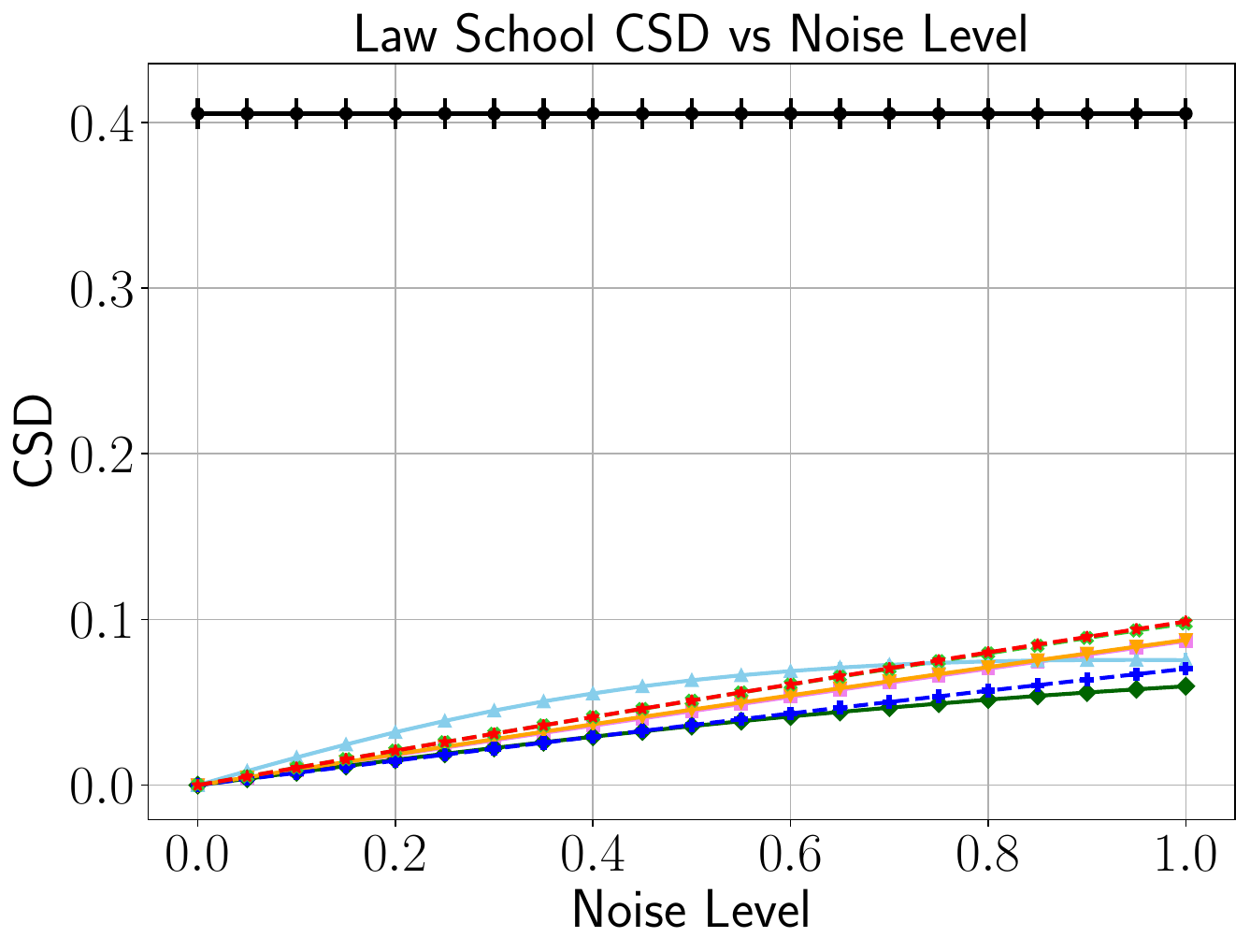}
        \caption{Law School Regression}
        \label{fig:law-reg-csd-noise}
    \end{subfigure}%
    \begin{subfigure}{0.25\textwidth}
        \centering
        \includegraphics[width=\textwidth]{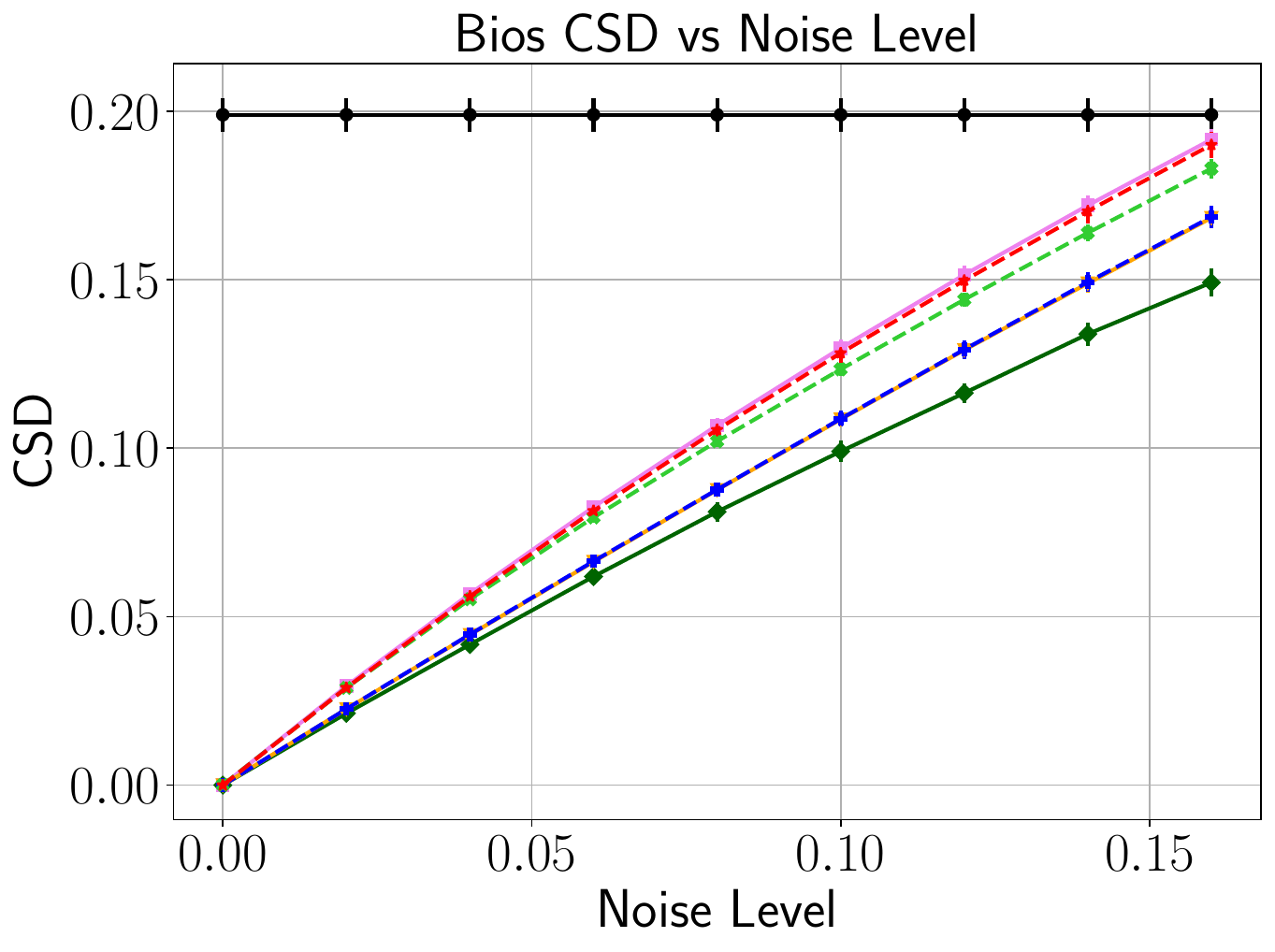}
        \caption{Bios Classification}
        \label{fig:bios-clf-csd-noise}
    \end{subfigure}
    \caption{Counterfactual Set Disparity (CSD) vs noise level in counterfactual generation. It shows that CF-CP methods show similar robustness to noise as other counterfactual fair models.}
    \label{fig:csd-vs-noise}
\end{figure*}

\subsection{Oracle Counterfactuals}
We first present results when the counterfactuals are generated from the true SCM (oracle counterfactuals)--later, we consider the case where noise is added to the exogenous variables.
Table~\ref{tab:synth-combined} shows the results on synthetic regression and classification datasets.
Table~\ref{tab:real-combined} shows the results on real datasets, Law School for regression and Bios for classification.
First, we observe that vanilla split CP achieves the target coverage but has high counterfactual set disparity (CSD) ranging from 0.19 to 0.71.
As expected, Post-hoc Union achieves counterfactual invariance for prediction sets (CSD=0), but at the cost of a significant increase in the size of the prediction sets and over-coverage due to inflation of the prediction sets.
Methods that train counterfactually fair predictors (CFU, CFR, PCF) achieve counterfactual fairness for both point predictions and prediction sets (CSD=0) but at the cost of lower accuracy/MSE and prediction sets that are moderately larger than split CP.
Our proposed CF-CP method achieves counterfactual fairness for prediction sets (CSD=0) while maintaining the accuracy/MSE of the base predictor, and it produces prediction sets that are similar in size to counterfactually fair baselines and moderately larger than split CP.
Only \texttt{max} aggregation leads to a nonzero CSD in synthetic classification, because of non-empty prediction set rule.
Also, CF-CP achieves comparable or better average set size compared to other counterfactual fairness methods (CFU, CFR, PCF) without requiring retraining or access to the probability distribution of $A$.
Among the three aggregation functions, the results are quite similar, with \texttt{mean} achieving the smallest average set size in most cases.

\subsection{Noisy Counterfactuals}
In this section, we examine the robustness of CF-CP to the quality of the available counterfactuals, since its performance necessarily depends on how accurate these counterfactuals are. 
Specifically, we study how fairness and performance change as we introduce noise into the counterfactuals.
For the datasets where we have access to the exogenous variables $U$ (synthetic datasets and Law School), we add Gaussian noise to the true exogenous variables before generating counterfactuals.
For the Bios dataset, we add Gaussian noise to the counterfactual embeddings.
The results are shown in Table~\ref{tab:noisy-combined}.
Here we only report average set size and CSD due to space constraints and the other metrics and noise variances can be found in Appendix~\ref{sec:additional-results}.
The noisy counterfactuals lead to a slight increase in CSD for all methods, but CF-CP still achieves a significant reduction in CSD compared to split CP, with only a modest increase in the size of the prediction sets.
CF-CP achieves comparable average set size and CSD trade-off compared to other counterfactual fairness methods (CFU, CFR, PCF) even with noisy counterfactuals.

To further illustrate the robustness of CF-CP to noisy counterfactuals, we plot CSD against different noise levels in Figure~\ref{fig:csd-vs-noise}.
We vary the noise level from 0 (oracle counterfactuals) to 1 for synthetic datasets and Law School, and from 0 to 0.16 for Bios, as higher noise levels lead to poor performance.
We observe that as the noise level increases, CSD also increases for all methods as expected.
However, CF-CP methods show similar robustness to noise as other counterfactual fair models (CFU, CFR, PCF) and post-hoc union, and they consistently achieve lower CSD compared to split CP across all noise levels.
This shows that CF-CP is effective in reducing counterfactual unfairness even when counterfactuals are noisy.

We emphasize two key takeaways: (i) CF-CP does not require retraining the model or access to the probability distribution of $A$, making it a more practical choice in real-world scenarios where counterfactuals may be noisy or imperfect; and (ii) as stated earlier, CF-CP does not reduce the accuracy of the baseline predictor. 
Compared to CFU, CFR and PCF, the baseline predictor for CF-CP is more accurate--see Appendix~\ref{sec:additional-results}.
\section{CONCLUSION}
\label{sec:conclusion}

In this paper, we extended the notion of counterfactual fairness to set-valued predictors and then proposed a novel post-training conformal prediction method that enforces set-level counterfactual fairness by symmetrizing the conformity scores across protected attribute interventions.
Under an invertible SCM, we showed that our method guarantees both marginal coverage and counterfactual invariance for prediction sets.
Empirically, across synthetic and real-world datasets, including Law School and Bios, CF-CP substantially reduces counterfactual set disparity compared to standard split CP, with only a modest increase in average set size, and without requiring retraining or access to the probability distribution of the protected attribute.
Our results suggest that enforcing counterfactual fairness at the set level is both feasible and effective.
Two promising directions for future work are: (i) extending this fairness notion to path-dependent counterfactual fairness~\citep{chiappa2019path}, where only certain causal pathways from the protected attribute to the outcome are considered unfair; (ii) generalizing CF-CP to dynamic settings where covariates and protected attributes evolve over time.

% \subsubsection*{Acknowledgements}

% add the bibliography files
\bibliography{cf_cp}

%%%%%%%%%%%%%%%%%%%%%%%%%%%%%%%%%%%%%%%%%%%%%%%%%%%%%%%%%%%%
\section*{Checklist}

\begin{enumerate}

  \item For all models and algorithms presented, check if you include:
  \begin{enumerate}
    \item A clear description of the mathematical setting, assumptions, algorithm, and/or model. \textbf{Yes}.
    \item An analysis of the properties and complexity (time, space, sample size) of any algorithm. \textbf{Yes}.
    \item (Optional) Anonymized source code, with specification of all dependencies, including external libraries. \textbf{Yes}.
  \end{enumerate}

  \item For any theoretical claim, check if you include:
  \begin{enumerate}
    \item Statements of the full set of assumptions of all theoretical results. \textbf{Yes}.
    \item Complete proofs of all theoretical results. \textbf{Yes}.
    \item Clear explanations of any assumptions. \textbf{Yes}.     
  \end{enumerate}

  \item For all figures and tables that present empirical results, check if you include:
  \begin{enumerate}
    \item The code, data, and instructions needed to reproduce the main experimental results (either in the supplemental material or as a URL). \textbf{Yes}.
    \item All the training details (e.g., data splits, hyperparameters, how they were chosen). \textbf{Yes}.
    \item A clear definition of the specific measure or statistics and error bars (e.g., with respect to the random seed after running experiments multiple times). \textbf{Yes}.
    \item A description of the computing infrastructure used. (e.g., type of GPUs, internal cluster, or cloud provider). \textbf{Yes}.
  \end{enumerate}

  \item If you are using existing assets (e.g., code, data, models) or curating/releasing new assets, check if you include:
  \begin{enumerate}
    \item Citations of the creator If your work uses existing assets. \textbf{Yes}.
    \item The license information of the assets, if applicable. \textbf{Not Applicable}.
    \item New assets either in the supplemental material or as a URL, if applicable. \textbf{Not Applicable}.
    \item Information about consent from data providers/curators. \textbf{Not Applicable}.
    \item Discussion of sensible content if applicable, e.g., personally identifiable information or offensive content. \textbf{Not Applicable}.
  \end{enumerate}

  \item If you used crowdsourcing or conducted research with human subjects, check if you include:
  \begin{enumerate}
    \item The full text of instructions given to participants and screenshots. \textbf{Not Applicable}.
    \item Descriptions of potential participant risks, with links to Institutional Review Board (IRB) approvals if applicable. \textbf{Not Applicable}.
    \item The estimated hourly wage paid to participants and the total amount spent on participant compensation. \textbf{Not Applicable}.
  \end{enumerate}

\end{enumerate}

\clearpage
\appendix
\thispagestyle{empty}
\onecolumn

\aistatstitle{Supplementary Materials}

% \section{ADDITIONAL RELATED WORKS}

\section{PROOFS}

\subsection{Proof of Lemma~\ref{lem:score-invariance}}
\label{sec:proof-lem-score-invariance}
Here we restate Lemma~\ref{lem:score-invariance} for convenience.
\begin{lemma}[Invariance of the symmetrized score]
\label{lem:score-invariance-app}
Let Assumption~\ref{assump:invertible} hold.
Fix any $y\in\mathcal{Y}$.
Then for all $a,a'\in\mathcal{A}$,
\begin{equation}
\label{eq:sym-score-invariance-app}
    s_{\mathrm{cf}}\left(X_{A\leftarrow a},a,y\right) = s_{\mathrm{cf}}\left(X_{A\leftarrow a'},a',y\right).
\end{equation}
\end{lemma}

\begin{proof}
Fix $y\in\mathcal{Y}$ and $a,a'\in\mathcal{A}$. Under Assumption~\ref{assump:invertible}, for the factual pair $(X_{A\leftarrow a},a)$ there exists a unique latent
\begin{equation}
    U = g_a^{-1}\left(X_{A\leftarrow a}\right),
\end{equation}
where $g_a^{-1}$ is the inverse of the forward map $g_a$ in the SCM that generates $X$ from $U$ under intervention $A\leftarrow a$.
Likewise for $(X_{A\leftarrow a'},a')$ we have the same $U=g_{a'}^{-1}(X_{A\leftarrow a'})$ because both arise from the same structural model applied to the same exogenous noise.

For any target intervention $a''\in\mathcal{A}$, the corresponding counterfactual features are
\begin{equation}
    X_{A\leftarrow a''} = g_{a''}(U),
\end{equation}
independently of whether we start from $(X_{A\leftarrow a},a)$ or $(X_{A\leftarrow a'},a')$.
Hence the multiset of per-intervention scores
\begin{equation}
    \left\{s\left(g_{a''}(U),a'',y\right) : a''\in\mathcal{A}\right\}
\end{equation}
is identical in both cases (up to permutation). Since $\mathrm{Agg}$ in~\eqref{eq:s-sym} is permutation-invariant, applying it to these equal multisets yields the same value:
\begin{equation}
    s_{\mathrm{cf}}\left(X_{A\leftarrow a},a,y\right)=s_{\mathrm{cf}}\left(X_{A\leftarrow a'},a',y\right).
\end{equation}
This proves \eqref{eq:sym-score-invariance-app}.
\end{proof}

\subsection{Proof of Theorem~\ref{thm:exch}}
\label{sec:proof-thm-exch}
Here we restate Theorem~\ref{thm:exch} for convenience.
\begin{theorem}[Exchangeability preservation]
\label{thm:exch-app}
Assume that Assumption~\ref{assump:invertible} holds and that the calibration points $\{(X_i,A_i,Y_i)\}_{i\in\mathcal{I}_{\mathrm{cal}}}$ and the test point $(X_{n+1},A_{n+1},Y_{n+1})$ are exchangeable.
Suppose the score function $s$ and the aggregator $\mathrm{Agg}$ are measurable, and $\mathrm{Agg}$ is permutation-invariant.
Define the counterfactual symmetrization
\begin{equation*}
    s_{\mathrm{cf}}(x,a,y) = \mathrm{Agg}\left\{ s\left(x_{A\leftarrow a'},a',y\right) : a'\in\mathcal{A}\right\}.
\end{equation*}
Then, the scores
\begin{equation}
    \{S^{\mathrm{cf}}_i = s_{\mathrm{cf}}(X_i,A_i,Y_i) : i \in \mathcal{I}_{\mathrm{cal}} \cup \{n+1\}\}
\end{equation}
are exchangeable.
\end{theorem}

\begin{proof}
Let $\pi$ be any permutation of $\mathcal{I}_{\mathrm{cal}}\cup\{n+1\}$.
By assumption, $\{(X_{\pi(i)},A_{\pi(i)},Y_{\pi(i)})\}_{i}$ is distributionally identical to $\{(X_i,A_i,Y_i)\}_{i}$.

Define the coordinate-wise transformation
\begin{equation}
    T(x,a,y) := \mathrm{Agg}\left\{ s\left(g_{a'}\big(g_a^{-1}(x)\big),a',y\right) : a'\in\mathcal{A} \right\}.
\end{equation}
Under Assumption~\ref{assump:invertible}, $g_a^{-1}$ and $g_{a'}$ are deterministic and measurable for all $a,a'$, and $s$ and $\mathrm{Agg}$ are measurable. 
Hence, $T$ is measurable, and it is applied identically to each coordinate.

Set
\begin{equation}
    S_i^{\mathrm{cf}} := T(X_i,A_i,Y_i).
\end{equation}
Exchangeability is preserved under identical measurable coordinate-wise transformations by Proposition~4 of \citet{kuchibhotla2020exchangeability}, it follows that
\begin{equation}
    \{S_i^{\mathrm{cf}}\}_{i} \overset{d}{=} \{s_{\mathrm{cf}}(X_{\pi(i)},A_{\pi(i)},Y_{\pi(i)})\}_{i}.
\end{equation}
Therefore, $\{S_i^{\mathrm{cf}}\}_{i\in\mathcal{I}_{\mathrm{cal}}\cup\{n+1\}}$ are exchangeable, as claimed.
\end{proof}

\section{EXPERIMENT DETAILS}
\label{sec:experiment-details}

\subsection{Implementation Details}
\label{sec:implementation-details}
Here we provide additional implementation details.
All GPU related experiments were run on GTX 1080 Ti.
Remaining experiments were run on Apple M1 CPU.

\paragraph{Base predictor.}
For all datasets, the base predictor $\hat{f}$ is a linear model for both regression and classification tasks.
For synthetic and Law School datasets, we use \texttt{scikit-learn}'s \texttt{LinearRegression} and \texttt{LogisticRegression} implementations~\citep{scikit-learn}.
The default hyperparameters are used except for the maximum number of iterations which is set to 1000 for \texttt{LogisticRegression}.
For Bias in Bios dataset, we use \texttt{PyTorch}~\citep{paszke2019pytorch} to implement the logistic regression model and utilize GPU to speed up optimization.
We use LBFGS optimizer with a learning rate of 1 and a maximum of 100 iterations.

\paragraph{Conformal prediction sets.}
When constructing conformal prediction sets, we include at least one label in the prediction set for classification tasks if the score of the most likely label is above the threshold.
This is to avoid empty prediction sets which can happen when the model is very uncertain about the prediction.
It is a common practice in conformal prediction to ensure non-empty prediction sets.

\subsection{Datasets}
\label{sec:datasets}

Here we introduce the datasets used in our experiments.

\paragraph{Synthetic Regression.}
We follow the data generation process in \citet{zuo2023counterfactually} to generate a synthetic regression dataset with a binary protected attribute $A$.
The data is generated according to the following structural equations:
\begin{align*}
    U_i &\sim \mathcal{N}(0,1)\ \ i\in{1,2}, \quad A \sim \text{Bernoulli}(0.4), \\
    X &\leftarrow \sin(U_1)+\cos(AU_2)+A + 0.1, \\
    Y &\leftarrow 0.2X^2+1.2X+0.2 + \epsilon_Y, \quad \epsilon_Y \sim \mathcal{N}(0,0.6).
\end{align*}
The causal graph is shown in Figure~\ref{fig:causal-graph-1}.
Counterfactuals change only $A$ in the $X$-equation (holding $U_1,U_2,\epsilon_Y$ fixed).
For noisy experiments, we add i.i.d. Gaussian noise to $U$ when constructing counterfactuals, i.e., we use $\tilde{U} = U + \epsilon_U$ where $\epsilon_U \sim \mathcal{N}(0,\sigma_U^2 I)$ and $\sigma_U=0.4$ for the results in Table~\ref{tab:noisy-combined}.
We use $\alpha=0.1$ and the split of the data is as follows: 5000 samples for training, 1000 samples for calibration, and 5000 samples for testing.

\begin{figure}[h]
    \centering
    \begin{subfigure}{0.32\textwidth}
        \centering
        \begin{tikzpicture}
            \node[state] (A) {$A$};
            \node[state, right of=A] (U) {$U$};
            \node[state, below of=A] (X) {$X$};
            \node[state, below of=U] (Y) {$Y$};
            \path (A) edge [] node [above] {} (X) ;
            \path (U) edge [] node [above] {} (X) ;
            \path (X) edge [] node [above] {} (Y) ;
        \end{tikzpicture}
        \caption{Synthetic Regression and Bias}
        \label{fig:causal-graph-1}
    \end{subfigure}
    \hfill
    \begin{subfigure}{0.32\textwidth}
        \centering
        \begin{tikzpicture}
            \node[state] (A) {$A$};
            \node[state, right of=A] (U) {$U$};
            \node[state, below of=A] (X) {$X$};
            \node[state, below of=U] (Y) {$Y$};
            \path (A) edge [] node [above] {} (X) ;
            \path (U) edge [] node [above] {} (X) ;
            \path (X) edge [] node [above] {} (Y) ;
            \path (U) edge [] node [above] {} (Y) ;
        \end{tikzpicture}
        \caption{Synthetic Classification}
        \label{fig:causal-graph-2}
    \end{subfigure}
    \hfill
    \begin{subfigure}{0.32\textwidth}
        \centering
        \begin{tikzpicture}
            \node[state] (R) {\small Race};
            \node[state, right of=R] (S) {Sex};
            \node[state, below of=R] (L) {LSAT};
            \node[state, below of=S] (G) {GPA};
            \node[state, below of=L, xshift=0.75cm] (F) {FYA};
            \path (R) edge [] node [above] {} (L) ;
            \path (R) edge [] node [above] {} (G) ;
            \path (R) edge [bend right=75] node [above] {} (F) ;
            \path (S) edge [] node [above] {} (L) ;
            \path (S) edge [] node [above] {} (G) ;
            \path (S) edge [bend left=75] node [above] {} (F) ;
            \path (L) edge [] node [above] {} (F) ;
            \path (G) edge [] node [above] {} (L) ; 
            \path (G) edge [] node [above] {} (F) ;

        \end{tikzpicture}
        \caption{Law School}
        \label{fig:causal-graph-3}
    \end{subfigure}
    \caption{Causal Graphs of the Datasets Used in Experiments}
    \label{fig:causal-graphs}
\end{figure}
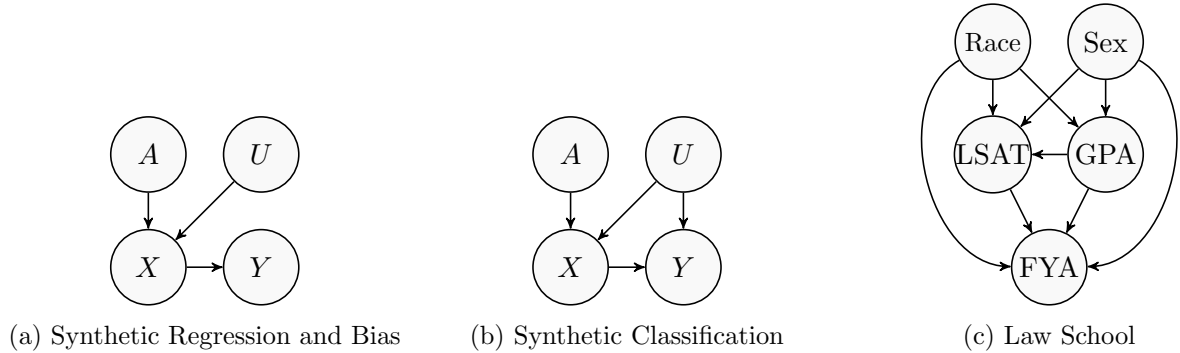

\paragraph{Synthetic Classification.}
We extend the binary classification setting in \citet{zhou2024counterfactual} to a multi-class classification setting with $K=10$ classes.
The SCM is
\begin{align*}
U &\sim \mathcal{N}(\mathbf{0}, I_d),\qquad A \sim \mathrm{Bernoulli}(0.5),\\
X &\leftarrow (A-0.5)w_A + U D_U,\\
Y &\sim \mathrm{Categorical}\left(\mathrm{softmax}(X^{\odot 3} W_X + U W_U + E)\right),\qquad E \sim \mathcal{N}(\mathbf{0}, \sigma^2 I_K).
\end{align*}
The causal graph is shown in Figure~\ref{fig:causal-graph-2}.
Here $d=10$ is the feature dimension and $K=10$ the number of classes; $w_A \in \mathbb{R}^{1\times d}$ encodes the linear effect of $A$ on $X$; $D_U \in \mathbb{R}^{d\times d}$ mixes the latent $U$ into $X$ (assumed full rank); $W_X,W_U \in \mathbb{R}^{d\times K}$ map the (nonlinear) features and latent into class logits; $E$ is i.i.d. Gaussian \emph{logit} noise with variance $\sigma=0.2$ and $X^{\odot 3}$ denotes elementwise cubing of $X$.
All the matrices are generated as identity matrix plus small uniform noise between $[0,0.2]$ added to each entry, $w_A$ has first $r=3$ entries sampled from $\mathrm{Uniform}([2, 2.2])$ and the rest are zeros.
Counterfactuals change only $A$ in the $X$-equation (holding $U$ and $E$ fixed).
Similarly, for noisy experiments, we add i.i.d. Gaussian noise to $U$ when constructing counterfactuals, i.e., we use $\tilde{U} = U + \epsilon_U$ where $\epsilon_U \sim \mathcal{N}(0,\sigma_U^2 I)$ and $\sigma_U=0.4$ for the results in Table~\ref{tab:noisy-combined}.
We use $\alpha=0.1$ and the split of the data is as follows: 5000 samples for training, 1000 samples for calibration, and 5000 samples for testing.

\paragraph{Law School.}
We use the Law School dataset from \citet{wightman1998lsac}, which has been used in prior work on counterfactual fairness~\citep{kusner2017counterfactual,zuo2023counterfactually}.
The task is to predict the first-year average grade (FYA) of law students based on their LSAT score, undergraduate GPA, race and gender.
We consider race as the protected attribute $A$ (binary: white vs non-white).
We preprocess the LSAT and GPA features by standardizing them to have zero mean and unit variance.
To construct counterfactuals, we fit a LiNGAM model to the data using the \texttt{causal-learn} package~\citep{zheng2024causal}.
The fitted causal graph is shown in Figure~\ref{fig:causal-graph-3}.
We assume that race and gender are root nodes in the causal graph and there are no children of FYA node.
For noisy experiments, we add i.i.d. Gaussian noise to $U$ when constructing counterfactuals, and we assume $U$ only effects LSAT and GPA, i.e., we use $\tilde{U} = U + \epsilon_U$ where $\epsilon_U \sim \mathcal{N}(0,\sigma_U^2 I)$ and $\sigma_U=0.4$ for the results in Table~\ref{tab:noisy-combined}.
We use $\alpha=0.1$ and the total number of samples is 21791, with a split of 1000 for calibration, 10000 for testing, and the rest for training.

\paragraph{Bias in Bios.}
We use the Bios dataset from \citet{de2019bias} which contains biographies of individuals in various professions.
The task is to predict the profession of an individual based on their biography text, it has gender labels which we consider as the protected attribute $A$ (binary: male vs female), there are 28 professions (classes) in total.
We get the embeddings of the factual and counterfactual biographies via the method proposed by~\citep{lemberger2024explaining}, using their code is publicly available.\footnote{\url{https://github.com/ToineSayan/counterfactual-representations-for-explanation}}
They assume a causal graph similar to Figure~\ref{fig:causal-graph-1} where $X$ is the text embedding and $Y$ is the profession, and $U$ is some latent variable that is not observed.
We use the version of the dataset introduced in \citet{ravfogel2020null}.
First we obtain text embeddings using a pre-trained BERT model~\citep{devlin2019bert}, then we split this embedding into two orthogonal components, one that is aligned with the protected attribute $A$ and one that is independent of $A$.
To create counterfactual embeddings, we flip the component aligned with $A$ while keeping the independent component fixed.
In this dataset, we do not have access to $U$, so we cannot apply CFU method.
For noisy experiments, we add i.i.d. Gaussian noise to the embedding when constructing counterfactuals, i.e., we use $\tilde{X}_{A\leftarrow a'} = X_{A\leftarrow a'} + \epsilon_X$ where $\epsilon_X \sim \mathcal{N}(0,\sigma_X^2 I)$ and $\sigma_X=0.1$ for the results in Table~\ref{tab:noisy-combined}.
We use $\alpha=0.025$ since the task is fairly easy, if we use $\alpha=0.1$ all methods return singletons as the prediction sets most of the time.
The total number of samples is 393423. 
We use 5000 samples for calibration, half of the dataset for testing, and the rest for training.

\section{ADDITIONAL RESULTS}
\label{sec:additional-results}

In this section, we provide additional experimental results that could not fit in the main text due to space constraints.

\subsection{Results on Noisy Counterfactuals}
\label{sec:additional-noisy-results}

Here we present additional results on the experiments with noisy counterfactuals.
In Table~\ref{tab:noisy-combined}, we report the results on all datasets with noisy counterfactuals.
We only report average set size and CSD due to space constraints, the other metrics are reported here in Table~\ref{tab:noisy-synth-combined} and Table~\ref{tab:noisy-real-combined}.
It can be seen that due to noisy counterfactuals, MSE or accuracy of counterfactually fair methods (CFU, CFR, PCF) are slightly worse than oracle counterfactuals.
This also leads to a slight increase in average set size for these methods.
However, counterfactually fair methods and CF-CP still achieve a significant reduction in CSD compared to split CP, with only a modest increase in the size of the prediction sets.
The noise level for synthetic and Law School datasets is $\sigma_U=0.4$, and for Bias in Bios dataset it is $\sigma_X=0.1$.

\begin{table*}[!t]
\centering
\caption{Synthetic results at $\alpha=0.1$ (mean $\pm$ std over 10 runs) with noisy counterfactuals. Left block: regression; right block: classification. Highlighting follows the convention described in the evaluation metrics paragraph.}
\label{tab:noisy-synth-combined}
\setlength{\tabcolsep}{4.5pt}
\renewcommand{\arraystretch}{1.05}
\resizebox{\textwidth}{!}{
\begin{tabular}{lccccc|ccccc}
\toprule
& \multicolumn{5}{c}{Regression} & \multicolumn{5}{c}{Classification} \\
\cmidrule(lr){2-6}\cmidrule(lr){7-11}
Method & MSE ($\downarrow$) & TE ($\downarrow$) & Coverage & Avg. Set Size ($\downarrow$) & CSD ($\downarrow$) & Accuracy ($\uparrow$) & TE ($\downarrow$) & Coverage & Avg. Set Size ($\downarrow$) & CSD ($\downarrow$) \\
\midrule
SplitCP & 0.377 $\pm$ 0.012 & 1.219 $\pm$ 0.014 & 0.901 $\pm$ 0.008 & 2.032 $\pm$ 0.038 & 0.713 $\pm$ 0.009 & 0.730 $\pm$ 0.011 & 0.543 $\pm$ 0.037 & 0.905 $\pm$ 0.011 & 2.025 $\pm$ 0.209 & 0.642 $\pm$ 0.043 \\
Post-hoc Union & 0.377 $\pm$ 0.012 & 1.219 $\pm$ 0.014 & 0.944 $\pm$ 0.006 & 3.259 $\pm$ 0.044 & 0.152 $\pm$ 0.001 & 0.730 $\pm$ 0.011 & 0.543 $\pm$ 0.037 & 0.935 $\pm$ 0.009 & 3.115 $\pm$ 0.254 & 0.260 $\pm$ 0.006 \\
CFU & 1.170 $\pm$ 0.020 & 0.267 $\pm$ 0.007 & 0.902 $\pm$ 0.008 & 3.551 $\pm$ 0.088 & 0.135 $\pm$ 0.004 & 0.512 $\pm$ 0.011 & 0.235 $\pm$ 0.004 & 0.908 $\pm$ 0.014 & 3.462 $\pm$ 0.239 & 0.345 $\pm$ 0.008 \\
CFR & 0.852 $\pm$ 0.014 & 0.191 $\pm$ 0.002 & 0.900 $\pm$ 0.012 & 3.013 $\pm$ 0.078 & \bftab{0.115 $\boldsymbol\pm$ 0.003} & 0.563 $\pm$ 0.012 & 0.140 $\pm$ 0.002 & 0.907 $\pm$ 0.014 & 2.913 $\pm$ 0.209 & \bftab{0.209 $\boldsymbol\pm$ 0.004} \\
PCF & 0.847 $\pm$ 0.014 & 0.189 $\pm$ 0.002 & 0.897 $\pm$ 0.010 & 2.983 $\pm$ 0.061 & \bftab{0.115 $\boldsymbol\pm$ 0.003} & 0.550 $\pm$ 0.015 & 0.165 $\pm$ 0.004 & 0.908 $\pm$ 0.012 & 2.600 $\pm$ 0.210 & 0.263 $\pm$ 0.006 \\
\midrule
CF-CP-mean & 0.377 $\pm$ 0.012 & 1.219 $\pm$ 0.014 & 0.900 $\pm$ 0.008 & 3.016 $\pm$ 0.054 & \dunderline{0.118 $\pm$ 0.002} & 0.730 $\pm$ 0.011 & 0.543 $\pm$ 0.037 & 0.908 $\pm$ 0.013 & 2.609 $\pm$ 0.203 & 0.259 $\pm$ 0.006 \\
CF-CP-max & 0.377 $\pm$ 0.012 & 1.219 $\pm$ 0.014 & 0.895 $\pm$ 0.009 & 3.374 $\pm$ 0.057 & 0.145 $\pm$ 0.003 & 0.730 $\pm$ 0.011 & 0.543 $\pm$ 0.037 & 0.945 $\pm$ 0.004 & 4.271 $\pm$ 0.185 & \dunderline{0.230 $\pm$ 0.004} \\
CF-CP-min & 0.377 $\pm$ 0.012 & 1.219 $\pm$ 0.014 & 0.905 $\pm$ 0.008 & 2.935 $\pm$ 0.050 & 0.168 $\pm$ 0.003 & 0.730 $\pm$ 0.011 & 0.543 $\pm$ 0.037 & 0.910 $\pm$ 0.012 & 2.657 $\pm$ 0.205 & 0.271 $\pm$ 0.005 \\
\bottomrule
\end{tabular}

}
\end{table*}

\begin{table*}[!t]
\centering
\caption{Real data results with noisy counterfactuals. Left block: regression on Law School dataset with $\alpha=0.1$; right block: classification on Bios dataset with $\alpha=0.025$ (mean $\pm$ std over 10 runs). Highlighting follows the convention described in the evaluation metrics paragraph.}
\label{tab:noisy-real-combined}
\setlength{\tabcolsep}{4.5pt}
\renewcommand{\arraystretch}{1.05}
\resizebox{\textwidth}{!}{
\begin{tabular}{lccccc|ccccc}
\toprule
& \multicolumn{5}{c}{Law School Regression} & \multicolumn{5}{c}{Bios Classification} \\
\cmidrule(lr){2-6}\cmidrule(lr){7-11}
Method &  MSE ($\downarrow$) & TE ($\downarrow$) & Coverage & Avg. Set Size ($\downarrow$) & CSD ($\downarrow$) & Accuracy ($\uparrow$) & TE ($\downarrow$) & Coverage & Avg. Set Size ($\downarrow$) & CSD ($\downarrow$) \\
\midrule
SplitCP & 0.758 $\pm$ 0.005 & 0.723 $\pm$ 0.019 & 0.898 $\pm$ 0.009 & 2.847 $\pm$ 0.070 & 0.405 $\pm$ 0.010 & 0.801 $\pm$ 0.001 & 0.094 $\pm$ 0.003 & 0.975 $\pm$ 0.003 & 3.311 $\pm$ 0.199 & 0.199 $\pm$ 0.005 \\
Post-hoc Union & 0.758 $\pm$ 0.005 & 0.723 $\pm$ 0.019 & 0.943 $\pm$ 0.007 & 3.571 $\pm$ 0.078 & \dunderline{0.036 $\pm$ 0.001} & 0.801 $\pm$ 0.001 & 0.094 $\pm$ 0.002 & 0.979 $\pm$ 0.002 & 3.991 $\pm$ 0.231 & 0.130 $\pm$ 0.002 \\
CFU & 0.835 $\pm$ 0.007 & 0.087 $\pm$ 0.003 & 0.900 $\pm$ 0.009 & 2.998 $\pm$ 0.063 & 0.055 $\pm$ 0.002 & - & - & - & - & - \\
CFR & 0.829 $\pm$ 0.007 & 0.045 $\pm$ 0.001 & 0.899 $\pm$ 0.009 & 2.984 $\pm$ 0.064 & \bftab{0.029 $\boldsymbol\pm$ 0.001} & 0.796 $\pm$ 0.001 & 0.045 $\pm$ 0.001 & 0.974 $\pm$ 0.003 & 3.450 $\pm$ 0.226 & \bftab{0.099 $\boldsymbol\pm$ 0.003} \\
PCF & 0.829 $\pm$ 0.007 & 0.056 $\pm$ 0.001 & 0.897 $\pm$ 0.010 & 2.961 $\pm$ 0.071 & 0.037 $\pm$ 0.001 & 0.792 $\pm$ 0.001 & 0.049 $\pm$ 0.001 & 0.975 $\pm$ 0.002 & 3.560 $\pm$ 0.197 & \dunderline{0.109 $\pm$ 0.002} \\
\midrule
CF-CP-mean & 0.758 $\pm$ 0.005 & 0.723 $\pm$ 0.019 & 0.899 $\pm$ 0.014 & 3.108 $\pm$ 0.108 & \bftab{0.029 $\boldsymbol\pm$ 0.001} & 0.801 $\pm$ 0.001 & 0.093 $\pm$ 0.003 & 0.975 $\pm$ 0.003 & 3.581 $\pm$ 0.204 & \dunderline{0.109 $\pm$ 0.002} \\
CF-CP-max & 0.758 $\pm$ 0.005 & 0.723 $\pm$ 0.019 & 0.900 $\pm$ 0.014 & 3.114 $\pm$ 0.112 & 0.041 $\pm$ 0.001 & 0.801 $\pm$ 0.000 & 0.094 $\pm$ 0.003 & 0.975 $\pm$ 0.003 & 3.813 $\pm$ 0.210 & 0.123 $\pm$ 0.002 \\
CF-CP-min & 0.758 $\pm$ 0.005 & 0.723 $\pm$ 0.019 & 0.899 $\pm$ 0.012 & 3.107 $\pm$ 0.093 & 0.041 $\pm$ 0.001 & 0.801 $\pm$ 0.001 & 0.093 $\pm$ 0.003 & 0.975 $\pm$ 0.002 & 3.599 $\pm$ 0.165 & 0.128 $\pm$ 0.003 \\
\bottomrule
\end{tabular}
}
\end{table*}

\subsection{Results on Synthetic Classification with Different Base Score Functions}
\label{sec:additional-synth-classification-results}

Here we provide additional results on synthetic classification with different base score functions.
In addition to LAC (used in the main text), we also consider the following two base score functions. 
The first one is Adaptive Prediction Sets (APS) score~\citep{romano2020classification}:
\begin{equation}
    s_{\mathrm{APS}}(x,a,y) = \sum_{y'\in\mathcal{Y}, \hat{f}_{y'}(x,a) > \hat{f}_{y}(x,a)} \hat{f}_{y'}(x,a).
\end{equation}
We use APS score without additional randomness for simplicity.
The second one is the regularized APS (RAPS) score, which adds a regularization term to the APS score to penalize large prediction sets~\citep{angelopoulos2021uncertainty}:
\begin{equation}
    s_{\mathrm{RAPS}}(x,a,y) = s_{\mathrm{APS}}(x,a,y) + \lambda (\text{rank}(\hat{f}_{y}(x,a)) - k_{\mathrm{reg}})_+,
\end{equation}
where $\lambda>0$ is the regularization parameter, $k_{\mathrm{reg}}$ is a threshold parameter, and $\text{rank}(\hat{f}(x,a,y))$ is the rank of the true label $y$ in the sorted list of predicted probabilities.
We set $\lambda=0.5$ and $k_{\mathrm{reg}}=2$ in our experiments.

\begin{table*}[t]
\centering
\caption{Synthetic classification results at $\alpha=0.1$ (mean $\pm$ std over 10 runs) with APS and RAPS base scores with noiseless counterfactuals. Highlighting follows the convention described in the evaluation metrics paragraph. Since we do not use random sets for APS and RAPS, we relax the condition of achieving target coverage.}
\label{tab:aps-raps-synth}
\setlength{\tabcolsep}{4.5pt}
\renewcommand{\arraystretch}{1.05}
\resizebox{\textwidth}{!}{
\begin{tabular}{lccccc|ccccc}
\toprule
& \multicolumn{5}{c}{APS} & \multicolumn{5}{c}{RAPS} \\
\cmidrule(lr){2-6}\cmidrule(lr){7-11}
Method &  Accuracy ($\uparrow$) & TE ($\downarrow$) & Coverage & Avg. Set Size ($\downarrow$) & CSD ($\downarrow$) & Accuracy ($\uparrow$) & TE ($\downarrow$) & Coverage & Avg. Set Size ($\downarrow$) & CSD ($\downarrow$) \\
\midrule
SplitCP & 0.730 $\pm$ 0.011 & 0.543 $\pm$ 0.037 & 0.962 $\pm$ 0.006 & 3.517 $\pm$ 0.169 & 0.583 $\pm$ 0.032 & 0.730 $\pm$ 0.011 & 0.543 $\pm$ 0.037 & 0.905 $\pm$ 0.011 & 2.393 $\pm$ 0.280 & 0.604 $\pm$ 0.048 \\
Post-hoc Union & \bftab{0.730 $\boldsymbol\pm$ 0.011} & 0.543 $\pm$ 0.037 & 0.979 $\pm$ 0.004 & 5.021 $\pm$ 0.238 & 0.000 $\pm$ 0.000 & \bftab{0.730 $\boldsymbol\pm$ 0.011} & 0.543 $\pm$ 0.037 & 0.940 $\pm$ 0.010 & 3.559 $\pm$ 0.316 & 0.000 $\pm$ 0.000 \\
CFU & \dunderline{0.589 $\pm$ 0.017} & 0.000 $\pm$ 0.000 & 0.932 $\pm$ 0.006 & 3.695 $\pm$ 0.126 & 0.000 $\pm$ 0.000 & \dunderline{0.589 $\pm$ 0.017} & 0.000 $\pm$ 0.000 & 0.907 $\pm$ 0.014 & 2.973 $\pm$ 0.183 & 0.000 $\pm$ 0.000 \\
CFR & \dunderline{0.589 $\pm$ 0.017} & 0.000 $\pm$ 0.000 & 0.931 $\pm$ 0.006 & 3.685 $\pm$ 0.134 & 0.000 $\pm$ 0.000 & \dunderline{0.589 $\pm$ 0.017} & 0.000 $\pm$ 0.000 & 0.907 $\pm$ 0.014 & 2.973 $\pm$ 0.182 & 0.000 $\pm$ 0.000 \\
PCF & 0.571 $\pm$ 0.017 & 0.000 $\pm$ 0.000 & 0.928 $\pm$ 0.008 & \bftab{3.487 $\boldsymbol\pm$ 0.145} & 0.000 $\pm$ 0.000 & 0.571 $\pm$ 0.017 & 0.000 $\pm$ 0.000 & 0.906 $\pm$ 0.014 & \bftab{2.859 $\pm$ 0.241} & 0.000 $\pm$ 0.000 \\
\midrule
CF-CP-mean & \bftab{0.730 $\boldsymbol\pm$ 0.011} & 0.543 $\pm$ 0.037 & 0.941 $\pm$ 0.004 & \dunderline{3.540 $\pm$ 0.119} & 0.017 $\pm$ 0.005 & \bftab{0.730 $\boldsymbol\pm$ 0.011} & 0.543 $\pm$ 0.037 & 0.927 $\pm$ 0.009 & 3.298 $\pm$ 0.160 & 0.025 $\pm$ 0.006 \\
CF-CP-max & \bftab{0.730 $\boldsymbol\pm$ 0.011} & 0.543 $\pm$ 0.037 & 0.947 $\pm$ 0.005 & 4.130 $\pm$ 0.127 & 0.028 $\pm$ 0.005 & \bftab{0.730 $\boldsymbol\pm$ 0.011} & 0.543 $\pm$ 0.037 & 0.942 $\pm$ 0.008 & 3.986 $\pm$ 0.147 & 0.034 $\pm$ 0.007 \\
CF-CP-min & \bftab{0.730 $\boldsymbol\pm$ 0.011} & 0.543 $\pm$ 0.037 & 0.944 $\pm$ 0.004 & 3.690 $\pm$ 0.139 & 0.018 $\pm$ 0.004 & \bftab{0.730 $\boldsymbol\pm$ 0.011} & 0.543 $\pm$ 0.037 & 0.912 $\pm$ 0.009 & \dunderline{2.948 $\pm$ 0.126} & 0.003 $\pm$ 0.004 \\
\bottomrule
\end{tabular}
}
\end{table*}

\begin{table*}[t]
\centering
\caption{Synthetic classification results at $\alpha=0.1$ (mean $\pm$ std over 10 runs) with APS and RAPS base scores with noisy counterfactuals, where the noise level is $\sigma_U=0.4$.  Highlighting follows the convention described in the evaluation metrics paragraph.}
\label{tab:aps-raps-synth-noisy}
\setlength{\tabcolsep}{4.5pt}
\renewcommand{\arraystretch}{1.05}
\resizebox{\textwidth}{!}{
\begin{tabular}{lccccc|ccccc}
\toprule
& \multicolumn{5}{c}{APS} & \multicolumn{5}{c}{RAPS} \\
\cmidrule(lr){2-6}\cmidrule(lr){7-11}
Method &  Accuracy ($\uparrow$) & TE ($\downarrow$) & Coverage & Avg. Set Size ($\downarrow$) & CSD ($\downarrow$) & Accuracy ($\uparrow$) & TE ($\downarrow$) & Coverage & Avg. Set Size ($\downarrow$) & CSD ($\downarrow$) \\
\midrule
SplitCP & 0.730 $\pm$ 0.011 & 0.543 $\pm$ 0.037 & 0.962 $\pm$ 0.006 & 3.517 $\pm$ 0.169 & 0.583 $\pm$ 0.032 & 0.730 $\pm$ 0.011 & 0.543 $\pm$ 0.037 & 0.905 $\pm$ 0.011 & 2.393 $\pm$ 0.280 & 0.604 $\pm$ 0.048 \\
Post-hoc Union & 0.730 $\pm$ 0.011 & 0.543 $\pm$ 0.037 & 0.979 $\pm$ 0.004 & 4.945 $\pm$ 0.232 & \dunderline{0.247 $\pm$ 0.008} & 0.730 $\pm$ 0.011 & 0.543 $\pm$ 0.037 & 0.938 $\pm$ 0.011 & 3.580 $\pm$ 0.322 & 0.258 $\pm$ 0.009 \\
CFU & 0.512 $\pm$ 0.011 & 0.235 $\pm$ 0.004 & 0.916 $\pm$ 0.012 & 4.007 $\pm$ 0.207 & 0.363 $\pm$ 0.010 & 0.512 $\pm$ 0.011 & 0.235 $\pm$ 0.004 & 0.904 $\pm$ 0.011 & 3.635 $\pm$ 0.239 & 0.343 $\pm$ 0.012 \\
CFR & 0.563 $\pm$ 0.012 & 0.140 $\pm$ 0.002 & 0.928 $\pm$ 0.008 & 3.790 $\pm$ 0.153 & \bftab{0.238 $\boldsymbol\pm$ 0.005} & 0.563 $\pm$ 0.012 & 0.140 $\pm$ 0.002 & 0.909 $\pm$ 0.012 & 3.219 $\pm$ 0.173 & \bftab{0.212 $\boldsymbol\pm$ 0.004} \\
PCF & 0.550 $\pm$ 0.015 & 0.165 $\pm$ 0.004 & 0.923 $\pm$ 0.010 & 3.484 $\pm$ 0.154 & 0.298 $\pm$ 0.008 & 0.550 $\pm$ 0.015 & 0.165 $\pm$ 0.004 & 0.907 $\pm$ 0.014 & 2.913 $\pm$ 0.244 & 0.278 $\pm$ 0.012 \\
\midrule
CF-CP-mean & 0.730 $\pm$ 0.011 & 0.543 $\pm$ 0.037 & 0.945 $\pm$ 0.005 & 3.634 $\pm$ 0.118 & 0.303 $\pm$ 0.011 & 0.730 $\pm$ 0.011 & 0.543 $\pm$ 0.037 & 0.932 $\pm$ 0.007 & 3.511 $\pm$ 0.112 & \dunderline{0.228 $\pm$ 0.006} \\
CF-CP-max & 0.730 $\pm$ 0.011 & 0.543 $\pm$ 0.037 & 0.952 $\pm$ 0.005 & 4.661 $\pm$ 0.203 & 0.254 $\pm$ 0.007 & 0.730 $\pm$ 0.011 & 0.543 $\pm$ 0.037 & 0.948 $\pm$ 0.006 & 4.410 $\pm$ 0.174 & 0.231 $\pm$ 0.008 \\
CF-CP-min & 0.730 $\pm$ 0.011 & 0.543 $\pm$ 0.037 & 0.947 $\pm$ 0.005 & 3.751 $\pm$ 0.126 & 0.313 $\pm$ 0.011 & 0.730 $\pm$ 0.011 & 0.543 $\pm$ 0.037 & 0.911 $\pm$ 0.010 & 2.991 $\pm$ 0.151 & 0.285 $\pm$ 0.017 \\
\bottomrule
\end{tabular}
}
\end{table*}

The results are shown in Table~\ref{tab:aps-raps-synth} for oracle counterfactuals and Table~\ref{tab:aps-raps-synth-noisy} for noisy counterfactuals.
The results are consistent with those using LAC score in the main text.
APS score tends to produce larger prediction sets compared to LAC score, while RAPS score produces smaller prediction sets thanks to the regularization.
Since we use deterministic version of APS, the average coverage is slightly larger than the target coverage value for all the methods.
Here again counterfactually fair methods (CFU, CFR, PCF) and CF-CP methods achieve a significant reduction in CSD compared to split CP, with only a modest increase in the size of the prediction sets.
In the noiseless case, CF-CP methods has CSD slightly larger than 0, which is due to the constraint that the prediction sets must include at least one label.
For completeness, we also report how CSD changes with different noise levels in Figure~\ref{fig:csd-vs-noise-aps-raps} for APS and RAPS scores.
The results are consistent with those using LAC score in the main text.
When the noise level is higher, all methods have higher CSD, but counterfactually fair methods and CF-CP methods still achieve a significant reduction in CSD compared to split CP.

\begin{figure*}[h]
    \centering
    \begin{subfigure}{0.5\textwidth}
        \centering
        \includegraphics[width=0.7\textwidth]{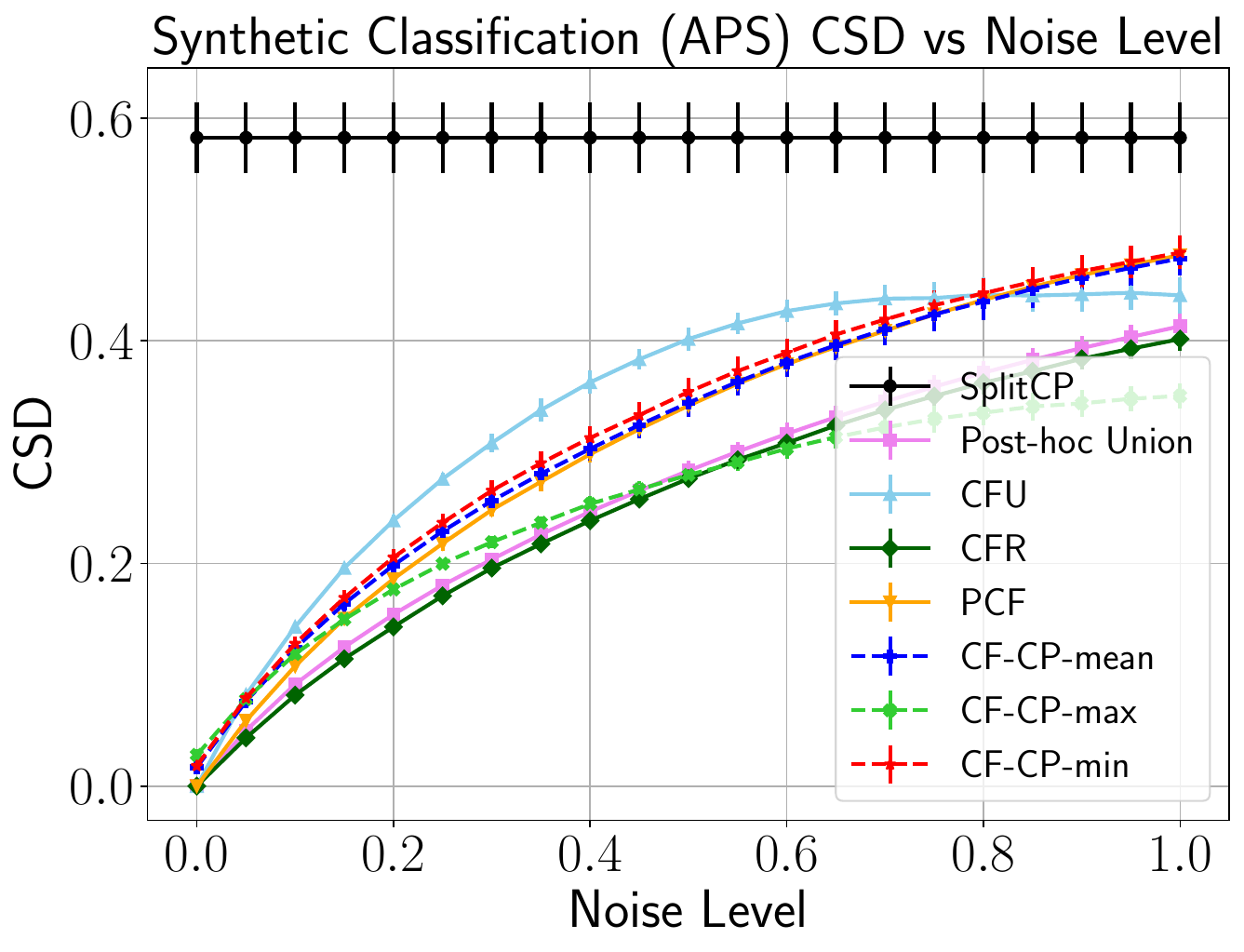}
        \caption{Synthetic Classification with APS}
        \label{fig:synth-aps-csd-noise}
    \end{subfigure}%
    \begin{subfigure}{0.5\textwidth}
        \centering
        \includegraphics[width=0.7\textwidth]{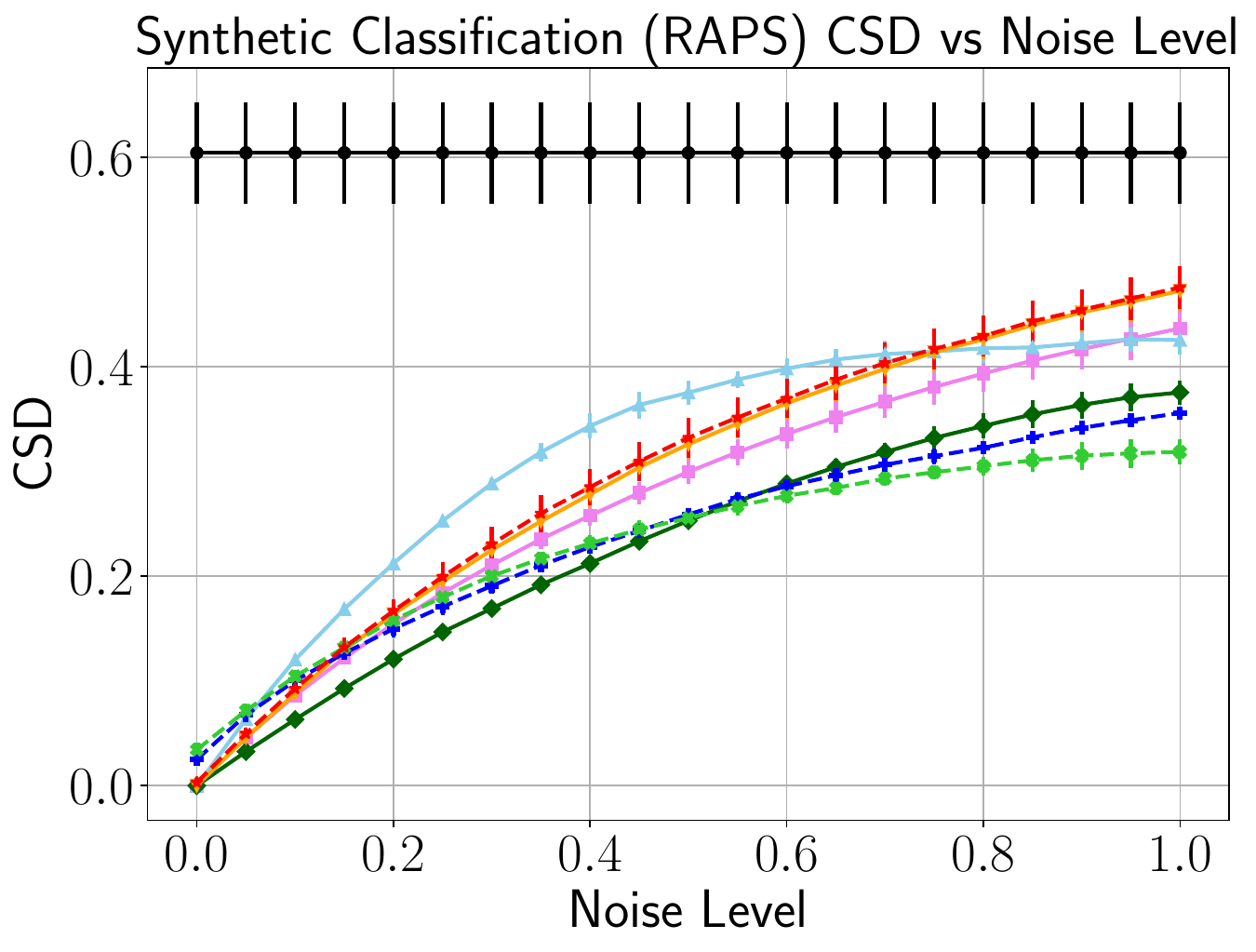}
        \caption{Synthetic Classification with RAPS}
        \label{fig:synth-raps-csd-noise}
    \end{subfigure}
    \caption{Counterfactual Set Disparity (CSD) vs noise level in counterfactual generation for APS(left) and RAPS (right) base scores on synthetic classification dataset (mean over 10 runs).}
    \label{fig:csd-vs-noise-aps-raps}
\end{figure*}

\end{document}